\documentclass[conference]{IEEEtran}
\IEEEoverridecommandlockouts

\PassOptionsToPackage{numbers, compress}{natbib}
\usepackage{natbib}
\usepackage[utf8]{inputenc} 
\usepackage[T1]{fontenc}    
\usepackage[colorlinks=true,allcolors=blue]{hyperref}
\usepackage{url}            
\usepackage{booktabs}       
\usepackage{amsfonts}       
\usepackage{nicefrac}       
\usepackage{microtype}      
\usepackage{xcolor}         

\usepackage{amssymb}
\usepackage{amsthm}
\usepackage{amsmath}

\usepackage{multirow}
\usepackage{arydshln}
\usepackage{algorithm}
\usepackage{algorithmic}

\usepackage{graphicx}
\usepackage{caption}
\usepackage{subcaption}

\newcommand{\rebuttal}[1]{#1}

\usepackage{amsthm}

\newtheorem{theorem}{Theorem}

\newtheorem{proposition}[theorem]{Proposition}

\newcommand{\mN}{\mathcal{N}}

\newcommand{\mR}{\mathcal{R}}

\newcommand{\mX}{\mathcal{X}}
\newcommand{\mC}{\mathcal{C}}

\newcommand{\pdata}{p_{\mathrm{data}}}

\newcommand{\Span}{\mathbf{span}} 

\begin{document}

\title{Data Redaction from Conditional Generative Models}

\author{\IEEEauthorblockN{1\textsuperscript{st} Zhifeng Kong}
\IEEEauthorblockA{\textit{Computer Science and Engineering} \\
\textit{University of California San Diego}\\
La Jolla, USA \\
\texttt{z4kong@ucsd.edu}}
\and
\IEEEauthorblockN{2\textsuperscript{nd} Kamalika Chaudhuri}
\IEEEauthorblockA{\textit{Computer Science and Engineering} \\
\textit{University of California San Diego}\\
La Jolla, USA \\
\texttt{kamalika@cs.ucsd.edu}}
}

\maketitle

\begin{abstract}
Deep generative models are known to produce undesirable samples such as harmful content. Traditional mitigation methods include re-training from scratch, filtering, or editing; however, these are either computationally expensive or can be circumvented by third parties.
In this paper, we take a different approach and study how to post-edit an already-trained
conditional generative model so that it redacts certain conditionals that will, with high probability, lead to undesirable content. This is done by distilling the conditioning network in the models, giving a solution that is effective, efficient, controllable, and universal for a class of deep generative models. We conduct experiments on redacting prompts in text-to-image models and redacting voices in text-to-speech models. Our method is computationally light, leads to better redaction quality and robustness than baseline methods while still retaining high generation quality. 
\end{abstract}

\section{Introduction}

Deep generative models are unsupervised deep learning models that learn a data distribution from samples and then generate new samples from it. These models have shown tremendous success in many domains such as image generation \citep{rombach2021highresolution, ramesh2021zero,ramesh2022hierarchical,sauer2023stylegan}, audio synthesis \citep{kong2021diffwave, lee2023bigvgan}, and text generation \citep{gpt4, touvron2023llama}. Most modern deep generative models are conditional: the user inputs some context known as the conditionals, and the model generates samples conditioned on the context.

However, as these models have grown more powerful, there has been increasing concern about their trustworthiness: in certain situations, these models produce undesirable outputs. For example, with text-to-image models, one may craft a prompt that contains offensive, biased, malignant, or fabricated content, and generate a high-resolution image that visualizes the prompt \citep{nichol2021glide, birhane2021multimodal, schuhmann2022laion, ramesh2022hierarchical,rando2022red, nudenet, man}. With speech synthesis models, one may easily turn text into celebrity voices \citep{Betker2022TTS,wang2023neural,zhang2023speak}. Text generation models can emit offensive, biased, or toxic content \citep{pitsilis2018detecting,wallace2019universal,mcguffie2020radicalization,gehman2020realtoxicityprompts,abid2021persistent,perez2022red,schramowski2022large}. 

One plausible solution to mitigate this problem is to remove all undesirable samples from the training set and re-train the model. This is too computationally heavy for modern, large models. 
Another solution is to apply a classifier that filters out undesirable conditionals or outputs \citep{rando2022red, nudenet, man}, or to edit the outputs and remove the undesirable content after generation \citep{schramowski2022safe}. However, in cases where the model owners share the model weights with third parties, they do not have control over whether the filters or editing methods will be used. 
In order to prevent undesirable outputs more efficiently and reliably, we propose to {\em{post-edit}} the weights of a pre-trained model, which we call \textit{data redaction}. 


The first challenge is how to frame data redaction for conditional generative models. Prior work in data redaction for \textit{unconditional} generative models considered this problem in the space of outputs, and framed the problem as learning the data distribution restricted to a valid subset of outputs \citep{kong2023data}. However, a conditional generative model learns a collection of (usually an infinite number of) distributions (one for each conditional) all of which are induced by networks that share weights; therefore, we cannot apply this method one by one for every conditional we would like to redact. 
In this paper, we frame data redaction for \textit{conditional} generative models as redacting a set of conditionals that will very likely lead to undesirable content. In particular, we do redaction in the conditional space, instead of separately redacting samples generated from each conditional in the output space.  

This statistical machine learning framework inspires us to design a \textit{universal}, \textit{efficient}, and \textit{effective} method for data redaction. We only re-train (or \textit{distill}) the conditional part of the network by projecting redacted conditionals onto different, non-redacted reference conditionals. It is computationally light because all but the conditioning network is fixed, and we only need to load a small fraction of the dataset for training.

We show there exists an explicit data redaction formula for simple class-conditional models. For more complicated generative models in real-world applications, we introduce a series of techniques to effectively redact certain conditionals but retain high generation quality. These include model-specific distillation losses and training schemes, methods to increase the capacity of the student conditioning network, ways to improve efficiency, and a few others. 

We test our data redaction method on two real-world applications: GAN-based text-to-image \citep{zhu2019dm} and Diffusion-based text-to-speech \citep{kong2021diffwave}. 
For text-to-image, we redact prompts that include certain words or phrases. Our method has significantly better redaction quality and robustness than baseline methods while retaining similar generation quality as the pre-trained model.
For text-to-speech, we redact certain voices outside the training set. Our method achieves both high redaction and speech quality. Audio samples can be found on our demo website (\url{https://dataredact2023.github.io/}). 
Our methods for both applications are extremely computationally efficient: redacting text-to-image models takes approximately 0.5 hour, and redacting text-to-speech models takes less than 4 hours, both on one single NVIDIA 3080 GPU. In contrast, training the text-to-image model takes more than a day on one GPU, and training the text-to-speech model takes 2-3 days on 8 GPUs. This demonstrates that data redaction can be done significantly more efficiently than re-training full models from scratch. 

\newcommand{\promptsetting}{none}
\newcommand{\promptbaseline}{none}
\newcommand{\promptid}{none}
\newcommand{\promptsubfig}[4]{
	\begin{subfigure}[b]{0.2\textwidth}
		\centering 
		\includegraphics[width=0.99\textwidth]{figs-#1-#2-#3}
		\caption{#4}
	\end{subfigure}
}

\renewcommand{\promptsetting}{base-long_beak_white_belly_to_short_beak_black_belly-invalid}
\renewcommand{\promptbaseline}{Rewrite-upsample1-long_beak_white_belly_to_short_beak_black_belly-invalid}
\renewcommand{\promptid}{031.Black_billed_Cuckoo_sentenceID227-0}

\begin{figure*}[!t]
\centering
    \promptsubfig{\promptsetting}{\promptid}{pretrained.png}{Pre-trained}
    \promptsubfig{\promptsetting}{\promptid}{target.png}{Reference}
    \promptsubfig{\promptsetting}{\promptid}{distilled.png}{Our Redaction}
    \promptsubfig{\promptbaseline}{\promptid}{distilled.png}{Rewriting \citep{bau2020rewriting}}
    \caption{Redact \texttt{``white belly''} from text-to-image models \citep{zhu2019dm}. The prompt  is \texttt{``this bird has feathers that are black and has a white belly''}. (a) Sample generated from the pre-trained model, which produces a visualization of the prompt. (b) The target sample that redacts \texttt{``white belly''} but keeps the other concepts. (c) Generated sample from our redaction model, which aims to redact \texttt{``white belly''} and approximates the reference sample. (d) Sample generated from the Rewriting baseline, which is blurry and has lower quality. More samples can be found in Appendix \ref{appendix: text2img exp: vis}.}
    \label{fig: page 2 figure}
\end{figure*}

\subsection{Related Work}

\textbf{Machine Unlearning.} Machine unlearning computes or approximates a re-trained machine learning model after removing certain training samples \citep{cao2015towards}. \rebuttal{Many unlearning methods have been proposed for supervised learning \citep{guo2019certified, schelter2020amnesia, neel2021descent, sekhari2021remember, izzo2021approximate, ullah2021machine, bourtoule2021machine,warnecke2021machine}, among which some provide theoretically guaranteed unlearning or removal for strictly convex classifiers. There is one method approximate deletion method for generative models \citep{kong2022approximate}, which aims to delete from an unconditional generative model by post-hoc rejection sampling}. The goal of data redaction is very different from machine unlearning, which unlearns training samples and is usually in the privacy context, while data redaction prevents undesirable samples from generation regardless whether they are in the training set. \footnote{\rebuttal{It is also unclear how to do efficient unlearning for complex conditional generative models (e.g. text-to-X) because it is unclear what exact combination of (text, X) pairs to unlearn and how to do it beyond retraining from scratch.}} A detailed explanation can be found in Section II-C in \citep{kong2023data}. 

\textbf{Data Filtering and Semantic Editing.} A direct way to prevent certain samples to be generated is to apply a data filter (e.g., a malicious content classifier). The filter can be applied to training data before training \citep{nichol2021glide, schuhmann2022laion, ramesh2022hierarchical}, or applied post-hoc to model outputs \citep{rando2022red, nudenet, man}.
Another line of research has looked at semantically modifying the outputs of generative models. For GANs \citep{goodfellow2014generative}, \citep{bau2020semantic} computes an editing vector in the latent space to alter a semantic concept. 
For diffusion models \citep{ho2020denoising} especially text-to-image models like Stable Diffusion \citep{rombach2021highresolution}, there are also a number of image editing techniques \citep{bar2022text2live, hertz2022prompt, kawar2022imagic, valevski2022unitune, brack2022stable}. \citep{schramowski2022safe} applied image editing to prevent diffusion models from generating malicious images through a safety guidance term that alters the sampling algorithm for inappropriate prompts. 

While these filtering and editing methods can be used to prevent malicious images, the model parameters are not modified. Consequently, in cases where the models owners share the model weights with third parties, they do not have control over whether the third parties will use the filters or editing methods. In contrast, our proposed method modifies the model weights to address this issue.

\textbf{Data Redaction in Unconditional Models.} 
Several works have studied methods to prevent generative models from producing undesirable samples, either by re-training or post-editing.
For GANs, \citep{asokan2020teaching} and \citep{sinha2021negative} investigated re-training methods via modified loss functions that penalize generation of undesirable samples, and \citep{bau2020rewriting} and \citep{cherepkov2021navigating} introduced post-hoc parameter rewriting techniques for semantic editing, which can be used to remove undesirable artifacts. \citep{kong2023data}, \citep{malnick2022taming}, and \citep{moon2023feature} designed post-editing data redaction methods for various types of pre-trained generative models. 

All these methods are restricted to the unconditional setting as they modify the mapping from latent vectors to samples. In contrast, the goal of this paper is to redact data from pre-trained, \textit{conditional} generative models. In these models, the conditional information heavily controls the content and style of generated samples (e.g. text-to-X), whereas the latent controls variation. It is therefore necessary to also modify the mapping from conditional to samples.

\textbf{Redaction Methods for Stable Diffusion.} \citep{gandikota2023erasing} fine-tunes Stable Diffusion to incorporate negative guidance on undesirable visual styles (e.g., those under copyright protection). As a result, undesirable samples will not be generated with the standard sampling algorithm. \rebuttal{However, one might recover the original score from the distilled score to break this method  (see Appendix \ref{appendix: recovering original score} for details).
\citep{gandikota2023unified} proposed an analytic solution for keys in cross attention blocks to edit concepts. A similar approach \citep{zhang2023forget} proposed a re-steering mechanism for keys by minimizing the attention maps of target concepts. 
\citep{heng2023selective} and \citep{kumari2023ablating} proposed to forget or manipulate concepts by further fine-tuning the entire network with certain continual learning objectives.
These methods are heavily designed for text-to-image tasks with Stable Diffusion. They require the model to be trained with either classifier-free guidance \citep{ho2022classifier} or cross attention blocks, or they need to fine-tune the entire large diffusion network. In contrast, our proposed method is universal, applies to a broader range of generative models, and applies to multiple data domains. To our knowledge, it is the first method that is able to redact voices from a trained speech synthesis model. 
}

\section{Preliminaries}

\textbf{Conditional Generative Models.}
Let $\mC$ be the space of conditionals. It could be a finite set of discrete labels, or an infinite set of continuous representations. 
\footnote{In cases where there are infinitely many discrete labels such as text or 16-bit floats, these conditionals are usually considered as continuous or transformed to continuous representations.}
For any $c\in\mC$ there is an underlying data distribution $\pdata(\cdot|c)$ (on $\mathbb{R}^d$) conditioned on $c$. In the discrete label case, this simply corresponds to a finite number of data distributions for all labels. In the more complicated continuous case, there is usually an underlying assumption that $\pdata(\cdot|c)$ is Lipschitz with respect to $c$: that is, $\pdata(\cdot|c)$ will not change much if $c$ does not change much.

Let $X=\{(x_i,c_i)\}$ be the set of training data, in which each $x_i$ is the sample and $c_i$ is the conditional (for example, $x_i$ is an image and $c_i$ is its caption). Let $G$ be a conditional generative model trained on $X$. $G$ has two inputs -- a sample latent $z$ drawn from a Gaussian distribution and a conditional $c$ -- and outputs sample $x=G(z|c)$. For each $c\in\mC$, $G$ draws from a generative distribution $p_G(\cdot|c)$, which is trained to learn $\pdata(\cdot|c)$. In the discrete label case, this is equivalent to modeling a finite number of distributions. In the continuous case, $G$ also needs to generalize to unseen conditionals, because not all conditionals exist in the training set. We assume that $p_G(\cdot|c)$ learns $\pdata(\cdot|c)$ very well, as how to train these models is outside the scope of this paper.

\textbf{Problem setup.} Our goal is to redact a set of conditionals $\mC_{\Omega}\subset\mC$, referred to as the redaction conditionals, which with high probability lead to undesirable content. For example, for text-to-image models, we may be looking to redact text prompts related to violence or offensive content. \footnote{This does not necessarily redact every possible offensive output; for example, an innocent prompt such as "a day in the park" might with very low probability result in a violent image which our solution will not address.} 

We assume that the redaction conditionals are given to us either as a set or described by a classifier. We assume that we are working with an already trained generative model $G$ and we are only allowed to post-edit it. Re-training generative models from scratch can be highly compute-intensive, and so our goal is to consider computationally efficient solutions. Additionally, we also want to avoid solutions that involve external filters, since a third-party can choose not to use them. A final requirement of our solution is that it should retain high generation quality for the conditionals that are not to be redacted. 

We assume that we have access to the parameters of the network $G$ and (part or whole of) its training dataset $X$. \footnote{\rebuttal{One setting this assumption holds is when the model owners want to make their model safer. We believe the only possible solution for a closed-source model is filtering. 
}} The goal of this paper is to edit the parameters of model $G$ to form a new model $G'$ so that harmful conditionals lead to the generation of benign outputs. 


Our proposed solution addresses this problem in the context where the conditioning networks are separate from the main generative network -- which holds for most current network architectures -- and achieves this by distilling only the conditioning networks. 

\section{Method}

In this section, we consider a special solution to our redaction task: for redacted conditionals $c\in\mC$, we let $G'$ learn the distribution conditioned on a different, non-redacted conditional $\hat{c}\in\mC\setminus\mC_{\Omega}$, which we denote as the reference conditional for $c$. Formally,
\begin{equation}\label{eq: general solution}
    p_{G'}(\cdot|c)=p_G(\cdot|\hat{c})\text{ if }c\in\mC_{\Omega}\text{, otherwise }p_G(\cdot|c).
\end{equation}

Next, we introduce an efficient way to achieve \eqref{eq: general solution}. Let $H$ be the (separate) conditioning network in the generator network $G$. $H$ takes the conditional $c$ as input and computes conditional representation $H(c)$, which is then fused into the main generative network (potentially at different layers). 
%
Our solution is to project the conditional representation $H'(c)$ of the new conditioning network to $H(\hat{c})$: 
\begin{equation}\label{eq: conditional solution}
    H'(c)=H(\hat{c})\text{ if }c\in\mC_{\Omega}\text{, otherwise }H(c).
\end{equation}

We provide an illustrative and analytical explanation to our method in Section \ref{sec: conditioned on discrete}. Specifically, we show \eqref{eq: conditional solution} can be done explicitly if the model is conditioned on a few discrete labels and the conditioning network is affine. 
We then introduce methods for more complicated, real-world scenarios in Section \ref{sec: conditioned on continuous}. For these models conditioned on continuous representations and with complicated architecture, we introduce distillation-based methods to approximately achieve \eqref{eq: conditional solution}.

~~\newline 

\section{Redacting Models Conditioned on Discrete Labels}\label{sec: conditioned on discrete}

In this section, we show for simple class-conditional models, there is an explicit formula to redact certain labels.

\textbf{Redacting a single label.}
Suppose there are $k$ labels: $\mC=\{c_1,\cdots,c_k\}$, where label $j$ is to be redacted. We consider a common type of conditioning method: each label $c_i$ is represented by a $k$-dimensional embedding vector $v_i\in\mathbb{R}^k$, and $H$ is an affine transformation whose output dimension $r\geq k$. We assume the embedding vectors are linearly independent: $\Span\{v_1,\cdots,v_k\}=\mathbb{R}^k$.
A special case of this formulation is the conditioning method proposed by \citep{mirza2014conditional}, where each $v_i=\mathbf{e}_i$ is the one-hot vector with the $i$-th element $=1$, and is concatenated to the latent code. 

Let $H(v)=Mv$, where $M\in\mathbb{R}^{r\times k}$. The redaction problem is equivalent finding an $M'\in\mathbb{R}^{r\times k}$ such that $M'v_i=Mv_i$ for $i\neq j$ and $M'v_j=MV_{-j}\eta_{-j}$ for an one-hot vector $\eta_{-j}\in\mathbb{R}^{k-1}$, where $V_{-j}=[v_1,\cdots,v_{j-1},v_{j+1},\cdots,v_k]\in\mathbb{R}^{k\times(k-1)}$.
The first condition $M'v_i=Mv_i$ for $i\neq j$ indicates every row of $M'-M$ is in the null space of $\{v_i\}_{i\neq j}$. The null space is a one-dimensional subspace with basis vector $u$. Then, $M'-M$ can be decomposed as $\omega u^{\top}$ for some $\omega\in\mathbb{R}^r$. Then, by to the second condition $M'v_j=MV_{-j}\eta_{-j}$, we have $\omega=\frac{1}{u^{\top}v_j}M(V_{-j}\eta_{-j}-v_j)$. 
This means by replacing $M$ with $M'=M(I+\frac{1}{u^{\top}v_j}(V_{-j}\eta_{-j}-v_j)u^{\top})$, we are able to redact label $j$. When conditioned on $j$, the edited model will generate another digit based on which element in $\eta_{-j}$ is non-zero. 

\textbf{Redacting multiple labels.} Suppose there are multiple labels $\{1,\cdots,J\}$ ($J<k$) to be redacted. The $M'$ matrix needs to satisfy $M'v_i=Mv_i$ for $i>J$, and $M'v_j=MV_{-J}\eta_{-j}$ for $j\leq J$, where $V_{-J}=[v_{J+1},\cdots,v_k]\in\mathbb{R}^{k\times(k-J)}$. For $j\leq J$, let $u_j$ be the basis vector of the null space of $\{v_i\}_{i\neq j}$. Each row of $M'-M$ is in the null space of $\{v_i\}_{i>J}$, which can be written as a linear combination of $\{u_j\}_{j=1}^J$. Therefore, we can represent $M'-M$ as
\[M'-M=\sum_{j=1}^J \omega_j u_j^{\top}=WU^{\top},\]
where the $j$-th column of $W$ ($U$) is $\omega_j$ ($u_j$). Let $V_J=[v_1,\cdots,v_J]$ and $Y_{-J}=[\eta_{-1},\cdots,\eta_{-J}]$. We have $M'V_J=MV_{-J}Y_{-J}$. This simplifies to
\[WU^{\top}V_J=M(V_{-J}Y_{-J}-V_J).\]
Notice that $U^{\top}V_J$ is a diagonal matrix with $j$-th diagonal element $u_j^{\top}v_j\neq0$. Therefore, we have 
\begin{equation}\label{eq: redact J discrete labels}
    W=M(V_{-J}Y_{-J}-V_J)(U^{\top}V_J)^{-1}.
\end{equation}

\textbf{Simplified formula for one-hot embedding vectors.} Let $v_i=\mathbf{e}_i$ for each $i$. Then, we have $u_i=v_i=\mathbf{e}_i$, and therefore $U^{\top}V_J=I$. We also have $U=V_J=[I_J|\mathbf{0}]^{\top}$ and $V_{-J}=[\mathbf{0}|I_{k-J}]$, where $I_J$ is the $J$-dimensional identity matrix. Then, 
\[WU^{\top}=M([\mathbf{0}|I_{k-J}]Y_{-J}-[I_J|\mathbf{0}]^{\top})[I_J|\mathbf{0}]=M\left(\begin{array}{cc}
    -I_J & \mathbf{0} \\
    Y_{-J} & \mathbf{0}
\end{array}\right).\]
As a result,
\[M'=M+WU^{\top}=M\left(\begin{array}{cc}
    \mathbf{0} & \mathbf{0} \\
    Y_{-J} & I_{k-J}
\end{array}\right).\]

\textbf{Higher embedding dimension.} Because of linear independence, the null space of $\{v_i\}_{i\neq j}$ has 1 dimension higher than the null space of $\{v_i\}_{i=1}^k$. Therefore, we can pick $u_j\in\mathbf{null}(\{v_i\}_{i\neq j})\setminus\mathbf{null}(\{v_i\}_{i=1}^k)$.

~~\newline

\section{Redacting Models Conditioned on Continuous Representations}\label{sec: conditioned on continuous}

In practice, the networks are usually complicated and highly non-linear. Therefore, there is generally no explicit formula to achieve \eqref{eq: conditional solution} due to non-linearity and limited expressive power of the conditioning network. To approximately achieve \eqref{eq: conditional solution}, we propose to distill the conditioning network by minimizing
\begin{equation}\label{eq: general loss}
\begin{array}{rl}
    \min_{H'}~L(H';\lambda) = & \mathbb{E}_{c\in\mC\setminus\mC_{\Omega}}\|H'(c)-H(c)\| \\
    & + \lambda\cdot \mathbb{E}_{c\in\mC_{\Omega}}\|H'(c)-H(\hat{c})\|
\end{array}
\end{equation}
for some metric $\|\cdot\|$ and balancing coefficient $\lambda>0$.
In the rest of this section, we study two types of common conditional generative models: image models conditioned on text prompts, and speech models conditioned on spectrogram representations. We will demonstrate specific losses and distillation techniques for each model that align with the slightly different goals in each task.

\subsection{Redacting GAN-based Text-to-Image Models}\label{sec: method t2img}

In this section, we study how to redact text prompts in text-to-image models. Modern text-to-image models can produce high-resolution images conditioned on text prompts that may be offensive, biased, malignant, or fabricated \citep{nichol2021glide, birhane2021multimodal, schuhmann2022laion, ramesh2022hierarchical,rando2022red, nudenet, man}. These models are usually expensive to re-train, so it is important to redact these prompts without re-training.

Especially, we look at DM-GAN \citep{zhu2019dm}, a GAN-based text-to-image model. It is trained on pairs of text and images from the CUB dataset \citep{CUB,reed2016learning}, a dataset for various species of birds. DM-GAN is composed of three cascaded generative networks $\{G_1,G_2,G_3\}$.
The first $G_1$ generates $64\times64$ images, the second $G_2$ up-samples to $128\times128$, and the third $G_3$ up-samples to $256\times256$. 
Each $G_i$ has its own conditioning network $H_i$. For a given prompt $c$, the model computes a sentence embedding $v_s(c)$ and word embeddings $v_w(c)$ from a pre-trained text encoder \citep{xu2018attngan}. The first conditioning network $H_1$ performs conditioning augmentation on the sentence embedding and concatenate the output to the latent variable. $H_2$ and $H_3$ apply memory writing modules to the word embeddings and fuse the outputs with the previously generated low-resolution images via several gates.

\textbf{Defining $\hat{c}$}.
We assume $\mC_{\Omega}$ contains prompts that have undesirable words or phrases. For these prompts, the reference prompts are defined by replacing these words with non-redacted ones.

\textbf{Sequential distillation.}
We propose to distill the conditioning networks $\{H_1,H_2,H_3\}$ sequentially based on \eqref{eq: general loss}. This is because both $G_2$ and $G_3$ are generative super-sampling networks, which take $G_1$ and $G_2$ outputs as inputs, respectively. After $G_1$ is edited to $G_1'$ for redaction, $G_2$ will take $G_1'$ outputs as inputs, and similar for $G_3$. Formally,
\begin{equation}\label{eq: t2img seq distill 1}
\begin{array}{rl}
    \displaystyle H_1'= \arg\min_{H_1'} & \{\mathbb{E}_{c\in\mC\setminus\mC_{\Omega}}\|H_1'(v_s(c))-H_1(v_s(c))\| \\
    & + \lambda\cdot\mathbb{E}_{c\in\mC_{\Omega}}\|H_1'(v_s(c))-H_1(v_s(\textcolor{red}{\hat{c}}))\|\},
\end{array}
\end{equation}

\begin{equation}\label{eq: t2img seq distill 2}
\begin{array}{rll} 
    \displaystyle H_i'=\arg\min_{H_i'} & \{\mathbb{E}_{c\in\mC\setminus\mC_{\Omega},z}& \|H_i'(v_w(c),G_{i-1}'(z|c)) - ~~~~~~ \\ 
    & & ~~~~~~ H_i(v_w(c),G_{i-1}'(z|c))\| \\ 
    \displaystyle + & \lambda\cdot\mathbb{E}_{c\in\mC_{\Omega},z} & \|H_i'(v_w(c),G_{i-1}'(z|c)) - ~~~~~~ \\
    & & ~~~~~~ H_i(v_w(\textcolor{red}{\hat{c}}),G_{i-1}'(z|\textcolor{red}{\hat{c}}))\|\} 
\end{array}
\end{equation}
for $i=2,3$.

\textbf{Improved capacity.} 
As $H_1'$ needs to approximate a piecewise function that is defined differently for two sets of sentence embeddings, we need to increase the capacity of $H_1'$ for better distillation. We append a few LSTM layers to the beginning of $H_1'$, which directly take the sentence embeddings as inputs. The LSTM layers are followed by a convolution layer that reduces hidden dimensions to 1. We initialize this layer with zero weights for training stability. We expect these layers can project sentence embeddings of $c$ to those of $\hat{c}$. The rest of $H_1'$ has the same architecture as $H_1$ but all weights are initialized for training.
We do not increase the capacity of $H_2'$ and $H_3'$ for two reasons. First, $H_1'$ has more direct impact on the generated images because it directly controls the initial low-resolution image. Second, the memory writing modules of $H_2$ and $H_3$ are already very expressive.

\textbf{Fixing the variance prediction part in $H_1$.} 
We aim to reduce the computational overhead by fixing certain variables.
The conditioning augmentation module in $H_1$ first computes a mean and a variance vector, and then samples from the Gaussian defined by them. We fix the variance prediction part and only distill the mean prediction part. In our experiments the number of parameters to be trained in $H_1'$ (with improved capacity) is reduced by $\sim32\%$ and therefore matches $H_1$.

\textbf{$\lambda$ annealing.}
In order to make sure the distilled conditioning networks also approximate the pre-trained ones well for non-redacted prompts, we anneal the balancing coefficient $\lambda$ during distillation: we initialize $\lambda=\lambda_{\min}$ and linearly increases to $\lambda_{\max}$ in the end.

\subsection{Redacting Diffusion-based Text-to-Speech Models}\label{sec: method tts}

Modern text-to-speech models can turn text into high-quality speech in unseen voices such as celebrity voices \citep{kong2021diffwave,Betker2022TTS,wang2023neural,zhang2023speak}. This may have unpredictable public impact if these models are used to fake celebrities. In this section, we study redacting certain voices from a pre-trained text-to-speech model.

Especially, we look at DiffWave \citep{kong2021diffwave}, a diffusion probabilistic model that is conditioned on spectrogram and outputs waveform. It is trained on speech of a single female reading a book, which we call the pre-trained voice. There are $n=30$ layers or residual blocks in DiffWave, each containing one independent conditioning network $H_i$. The architecture of each $H_i$ includes two up-sampling layers followed by one convolution layer. 

\textbf{Defining $\hat{c}$ with voice cloning.}
We assume $\mC_{\Omega}$ contains a few clips of speech in a specific voice. We train a voice cloning model (CycleGAN-VC2 \citep{kaneko2019cyclegan}) between the specific and pre-trained voices, and then transform all clips in $\mC_{\Omega}$ to the pre-trained voice. By doing this we obtain time-aligned pairs between $c\in\mC_{\Omega}$ and the corresponding $\hat{c}$: when we select a small duration $[t,t+\Delta t]$, the content of $c_{t:t+\Delta t}$ is the same as $\hat{c}_{t:t+\Delta t}$, yet only the voices are different.

\textbf{Improved voice cloning.}
We find the voice cloning quality of CycleGAN-VC2 can be improved by making the two unpaired training sets more similar. We first use a pre-trained Whisper model \citep{radford2022robust} to extract text from redacted speech. Then, we use Tortoise-TTS \citep{Betker2022TTS} to turn these text into speech in the pre-trained voice. Note that this cannot be used to define $\hat{c}$ directly because the generated samples are not time-aligned with the speech to be redacted. However, these generated samples are more similar to the redacted samples because they have the same text, and therefore it is easier for CycleGAN-VC2 to learn transformations between these two voices.

\textbf{Parallel distillation.} 
We propose to distill all conditional layers $H_i$'s in parallel as they are independent. We minimize the following loss:
    $\min \frac1n\sum_{i=1}^n L(H_i';\lambda).$

\textbf{Fixing up-sampling layers in $H_i$.}
To reduce computation overhead we fix the two up-sampling layers in each $H_i$. We only distill the last convolution layer in each $H_i$. 

\textbf{Improved capacity.}
To improve redaction quality, we increase the capacity of each $H_i'$ by replacing its last convolution layer $h_{\mathrm{conv}}$ with a spectrogram-rewriting module. It has two components: a gate $h_{\mathrm{gate}}$ consisting of a convolution with zero initialization followed by sigmoid, and a transformation block $h_{\mathrm{trans}}$ consisting of two convolution layers. The forward computation of the spectrogram-rewriting module is defined as:
    $y = h_{\mathrm{conv}}(v)\odot h_{\mathrm{gate}}(v) + h_{\mathrm{conv}}(h_{\mathrm{trans}}(v))\odot (1-h_{\mathrm{gate}}(v)),$
where $v$ is the up-sampled mel-spectrogram and $y$ is the output representation at each layer. We expect this module can retain the pre-trained voice and also project redacted voices to the pre-trained voice.

\textbf{Non-uniform distillation losses.}
We conjecture the all conditioning layers are not of the same importance because of their order and different hyper-parameters specifically the dilation $2^{i\mathrm{~mod~}n'}$ in the corresponding residual layer. This motivates us to use different weights and $\lambda$ values for each $H_i$:
    $\min \sum_{i=1}^n w_i L(H_i';\lambda_i).$
We test different schedules described in Table \ref{tab: tts distill non-uniform schedules}.

\begin{table}[!h]
    \centering
    \caption{Schedules for the non-uniform distillation losses.}
    \begin{tabular}{c|c}
    \toprule
        name & schedule \\ \hline
        $w_i$-order & $w_i=\frac1n+\alpha(i-(n+1)/2)$ \\
        $\lambda_i$-order & $\lambda_i=\lambda+\beta(i-(n+1)/2)$ \\ 
        $w_i$-dilation & $w_i=\frac1n+\alpha(i\mathrm{~mod~}n'-(n'+1)/6)$ \\
        $\lambda_i$-dilation & $\lambda_i=\lambda+\beta(i\mathrm{~mod~}n'-(n'+1)/6)$ \\
    \bottomrule
    \end{tabular}
    \label{tab: tts distill non-uniform schedules}
\end{table}

\section{Experiments}

In this section, we aim to answer the following questions.
(1) Is the redaction method in Section \ref{sec: conditioned on discrete} able to fully redact labels?
And (2) do the redaction algorithms in Section \ref{sec: conditioned on continuous} redact certain conditionals well and retain high generation quality on real-world applications?

\subsection{Redacting Models Conditioned on Discrete Labels}

We train a class-conditional GAN called cGAN \citep{mirza2014conditional} on MNIST \citep{lecun2010mnist}. Each conditional has a $10$-dimensional embedding vector, and is concatenated to the latent vector as the input. The affine transformation matrix $M$ in Section \ref{sec: conditioned on discrete} is the last 10 rows of the weight matrix of the first fully connected layer. We redact labels \texttt{0,1,2,3} according to \eqref{eq: redact J discrete labels}, where we let $\hat{c}=9-c$ for them. Generated samples of pre-trained and redacted models are shown in Fig. \ref{fig: cGAN visualization}.

\begin{figure}[!h]
    \centering
    \includegraphics[width=0.4\textwidth]{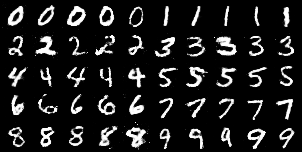} \\
    \vspace{0.2em}
    \includegraphics[width=0.4\textwidth]{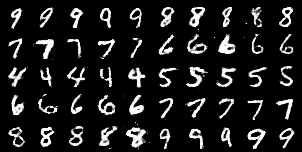}
    \caption{Redacting labels \texttt{0,1,2,3} in cGAN on MNIST. Upper: samples generated from the pre-trained model. Down: samples generated from the redacted model. Redacted conditionals (first two rows) are edited as expected, and other conditionals (last three rows) remain unchanged.}
    \label{fig: cGAN visualization}
\end{figure}

\subsection{Redacting GAN-based Text-to-Image Models}

\textbf{Setup.}
We use the pre-trained DM-GAN \citep{zhu2019dm} model trained on the CUB dataset \citep{CUB}, which contains 8855 training images and 2933 testing images of 200 subcategories belonging to birds. Each image has 10 captions \citep{reed2016learning}. Our distillation algorithm is trained with the caption data only. 
We redact prompts that contain certain words or phrases. We redact the word \texttt{blue}$\in c$ by defining $\hat{c}$ as the prompt that replaces all \texttt{blue} with another word \texttt{red}. \footnote{Any word other than \texttt{blue} can be used.} Similarly, we redact \texttt{blue wings} and \texttt{red wings} by replacing these phrases to \texttt{white wings}. We redact \texttt{long beak} and \texttt{white belly} by replacing the first to \texttt{short beak} and the second to \texttt{black belly}. Finally, we redact \texttt{yellow} and \texttt{red} by replacing them to \texttt{black}, which is more challenging as many samples are redacted. 

Table \ref{tab: t2img redacted num} includes the number of training and test prompts that are redacted in each experiment. Note that when we redact \texttt{blue wings} and \texttt{red wings}, we also redact phrases \texttt{wings that are blue} and \texttt{wings that are red}. 

\begin{table*}[!h]
    \centering
    \caption{Number of redacted training and test prompts. There are 88550 training prompts and 29330 test prompts in total.}
    \begin{tabular}{c|cc}
    \toprule
        Redaction prompts & \# redacted training prompts & \# redacted test prompts  \\ \hline
        \texttt{long beak, white belly} & 10377 & 3369 \\
        \texttt{blue / red wings} & 732 & 303 \\
        \texttt{blue} & 6113 & 2175 \\
        \texttt{yellow, red} & 29514 & 9319 \\
    \bottomrule
    \end{tabular}
    \label{tab: t2img redacted num}
\end{table*}

\textbf{Architecture and optimization.}
The architecture of the pre-trained model and other details are in Appendix \ref{appendix: text2img exp}. The architecture of student conditioning networks with improved capacity is shown in Fig. \ref{fig: DMGAN H}.
For each $H_i$, $i=1,2,3$, we use the Adam optimizer \citep{kingma2014adam} with a learning rate $0.005$ to optimize the mean square error loss. The redaction algorithm terminates at 1000 iterations. For $H_1$ we use a batch size of 128, and for $H_2$ and $H_3$ we reduce the batch size to 32 in order to fit into GPU memory. 

\textbf{Configurations.}
We first use the sequential distillation \eqref{eq: t2img seq distill 1} and \eqref{eq: t2img seq distill 2} with $\lambda=1$ to perform redaction, which we denote as the base configuration. We then improve the capacity by using a 3-layer bidirectional LSTM with hidden size $=32$ and dropout rate $=0.1$. Next, we fix the variance prediction in $H_1$ to reduce the number of parameters to optimize, which matches the base configuration. Finally, we apply $\lambda$ annealing by setting $\lambda_{\min}=1$ and $\lambda_{\max}=3$.

\textbf{Baseline.}
We compare to the Rewriting algorithm \citep{bau2020rewriting}, a semantic editing method originally designed for unconditional generative models. We adapt their method to DM-GAN by rewriting $G_1$, $G_2$, and $G_3$ sequentially. For both $G_2$ and $G_3$ we rewrite the up-sampling layer before the feature output. For $G_1$ we have choices of rewriting the up-sampling layer at different resolutions ranging from $8\times8$ to $64\times64$. We test all these choices in the experiment.

\textbf{Evaluation metrics.} 
\rebuttal{To evaluate \textit{generation quality} of $G'$, we compute Inception Scores (IS) \citep{salimans2016improved} for images conditioned on redacted and valid prompts, separately. In detail, the IS scores are computed as
$\exp(\mathbb{E}_{x}\mathbb{KL}(p(y|x)\parallel p(y)))$, where $x\sim p_{G'}(\cdot|c)$ for $c\sim\mathrm{Uniform}(\mC_{\Omega})$ or $\mathrm{Uniform}(\mC\setminus\mC_{\Omega})$, $p(y|x)$ is the logit from the Inception-V3 output layer \citep{szegedy2016rethinking}, and $p(y)$ is the marginal. We generate one sample for each text prompt for evaluation.
}

To evaluate \textit{redaction quality}, we compute the following three metrics where $c\sim\mC_{\Omega}$ and $z\sim\mN$. 
\begin{enumerate}
    \item $\mR_{G(\cdot|c/\hat{c})}$ measures faithfulness of $G'$ on the redaction prompts. It is defined as the fraction of samples $\{G'(z|c)\}$ such that $\mathrm{dist}(G'(z|c),G(z|\hat{c}))<\mathrm{dist}(G'(z|c),G(z|c))$, where $\mathrm{dist}$ is $\ell_2$ distance in the Inception-V3 feature space \citep{szegedy2016rethinking}.
    \item \rebuttal{A modified R-precision score $\mR_{r}$ measures how well $G'(z|c)$ matches the target caption $\hat{c}$. \citep{xu2018attngan} defined correlation $\mathrm{corr}(x,c)$ between sample $x$ and caption $c$ as $\cos\langle\mathrm{Enc}_{\mathrm{CNN}}(x),\mathrm{Enc}_{\mathrm{RNN}}(c)\rangle$ for pretrained CNN (image) and RNN (text) encoders. We use the pretrained encoders from DM-GAN.} 
    Then, $\mR_{r}$ is defined as the fraction of samples $G'(z|c)$ such that $\mathrm{corr}(G'(z|c),\hat{c})$ is larger than the correlation between $G'(z|c)$ and 100 random, mismatch captions.
    \item We further introduce $\mR_{c/\hat{c}}$, which measures how much better $G'(z|c)$ matches $\hat{c}$ than $c$. It is defined as the fraction of samples $G'(z|c)$ such that $\mathrm{corr}(G'(z|c),\hat{c}) > \mathrm{corr}(G'(z|c),c)$.
\end{enumerate}

\textbf{Results.}
The results for redacting \texttt{yellow} and \texttt{red} shown in Table \ref{tab: GAN redact config}. The base configuration already achieves good redaction and generation quality. After improving capacity, we find all redaction quality metrics increase by $2.3\sim2.7\%$, and generation quality is retained. After we fix the variance prediction in $H_1$, the redaction decrease by $\sim1\%$, but the generation quality on valid prompts increases by $0.1$. Finally, by performing $\lambda$ annealing, all metrics improve. Notably, $\mR_{G(\cdot|c/\hat{c})}$ and $\mR_{c/\hat{c}}$ increase by over $5\%$, indicating generated samples are more similar to $\hat{c}$ rather than $c$. 

We find the Rewriting baselines achieve better IS. However, generated samples are blurred and lack sharp edges as shown in the visualization. The redaction quality of Rewriting has a significant gap with ours: all redaction metrics are less than half of ours. Especially, $\mR_{r}$ is worse than the pre-trained model, indicating generated samples conditioned on redacted prompts are not very correlated to $\hat{c}$. We hypothesize the main problem for Rewriting is that it is crafted for 2D convolutions and edits the main generative network, which makes it hard to handle and distinguish the information from different prompts. In terms of different choices of resolutions, we find rewriting the layer at resolution $8\times8$ yields the best redaction quality.

\begin{table*}[!t]
    \centering
    \caption{Generation and redaction quality after redacting \texttt{yellow} and \texttt{red}. Our method achieves significantly better redaction quality than Rewriting and retains good generation quality. The effects of each component within our method are displayed.}
    
    \begin{tabular}{lc|cc|ccc|c}
    \toprule 
    
    \multicolumn{2}{c|}{\multirow{2}{*}{Method}} & \multicolumn{2}{c|}{Inception Score ($\uparrow$)} & \multicolumn{3}{c|}{Redacting quality ($\uparrow$)} & \rebuttal{Training time} \\ 
    & & redacted & valid & $\mR_{G(\cdot|c/\hat{c})}$ & $\mR_{c/\hat{c}}$ & $\mR_{r}$ & \rebuttal{mins} \\ \hline

    \multicolumn{2}{l|}{Pre-trained} & $4.62$ & $5.22$ & $0\%$ & $6.0\%$ & $13.5\%$ & - \\ \hline

    \multirow{4}{*}{Rewriting} & $8\times8$ & $5.57$ & $5.52$ & $\textit{33.0}\%$ & $\textit{39.7}\%$ & $\textit{5.0}\%$ & \rebuttal{$24.3$} \\
    & $16\times16$ & $5.63$ & $5.53$ & $30.4\%$ & $37.2\%$ & $4.8\%$ & \rebuttal{$25.3$} \\
    & $32\times32$ & $5.72$ & $5.71$ & $28.8\%$ & $35.9\%$ & $4.7\%$ & \rebuttal{$23.5$} \\
    & $64\times64$ & $\mathbf{5.77}$ & $\mathbf{5.73}$ & $27.5\%$ & $35.2\%$ & $4.6\%$ & \rebuttal{$24.1$} \\
    \hline
    
    \multicolumn{2}{l|}{Ours (base)} & $4.79$ & $5.23$ & $65.1\%$ & $77.0\%$ & $46.9\%$ & \rebuttal{$27.4$} \\ 
    \multicolumn{2}{l|}{~~+ improved capacity} & $4.74$ & $5.25$ & $67.8\%$ & $79.7\%$ & $\mathbf{49.2}\%$ & \rebuttal{$28.4$} \\ 
    \multicolumn{2}{l|}{~~~~+ fix variance} & $4.79$ & $5.35$ & $66.5\%$ & $79.0\%$ & $48.4\%$ & \rebuttal{$22.5$} \\ 
    \multicolumn{2}{l|}{~~~~~~+ $\lambda$ annealing} & $\textit{4.84}$ & $\textit{5.36}$ & $\mathbf{72.2}\%$ & $\mathbf{84.2}\%$ & $\mathbf{49.2}\%$ & \rebuttal{$29.9$} \\ 
    
    \bottomrule
    \end{tabular}
    \label{tab: GAN redact config}
\end{table*}

Table \ref{tab: GAN redact base} includes results for redacting the other prompts. The Rewriting baseline is applied to $8\times8$ resolution in $H_1$ because it yields the best redaction quality. We find the base configuration of our method is already very effective. Our method greatly outperforms Rewriting in all redaction quality metrics and keeps good generation quality. 

\textbf{Visualization}. See Appendix \ref{appendix: text2img exp: vis} for generated samples. The Rewriting baseline generate very blurry samples while our method generates high-quality, sharp samples which also satisfy the redaction requirements. 

\textbf{Computation}.
Data redaction takes about 30 minutes to train on a single NVIDIA 3080 GPU. 

\begin{table*}[!t]
    \centering
    \caption{Generation and redaction quality after redacting various words or phrases. Our method achieves significantly better redaction quality than Rewriting and retains good generation quality.}
    
    \begin{tabular}{cc|cc|ccc|c}
    \toprule 
    
    \multirow{2}{*}{Redaction prompts} & \multirow{2}{*}{Method} & \multicolumn{2}{c|}{Inception Score ($\uparrow$)} & \multicolumn{3}{c|}{Redacting quality ($\uparrow$)} & \rebuttal{Training time} \\ 
    & & redacted & valid & $\mR_{G(\cdot|c/\hat{c})}$ & $\mR_{c/\hat{c}}$ & $\mR_{r}$ & \rebuttal{mins}\\ \hline
    
    \multirow{3}{*}{\texttt{~long beak, white belly}} 
    & Pre-trained & $4.14$ & $5.61$ & $0\%$ & $5.2\%$ & $13.1\%$ & - \\
    & Rewriting & $\mathbf{5.36}$ & $\mathbf{5.85}$ & $32.6\%$ & $51.4\%$ & $5.6\%$ & \rebuttal{$23.0$} \\ 
    & Ours (base) & $4.91$ & $5.81$ & $\mathbf{70.5}\%$ & $\mathbf{83.6}\%$ & $\mathbf{50.1}\%$ & \rebuttal{$28.3$} \\ 
    \hline
    
    \multirow{3}{*}{\texttt{blue / red wings}}
    & Pre-trained & $3.97$ & $5.48$ & $0\%$ & $4.1\%$ & $13.1\%$ & - \\
    & Rewriting & $\mathbf{5.21}$ & $\mathbf{5.85}$ & $27.8\%$ & $15.1\%$ & $6.9\%$ & \rebuttal{$23.3$} \\ 
    & Ours (base) & $5.04$ & $5.28$ & $\mathbf{68.6}\%$ & $\mathbf{71.7}\%$ & $\mathbf{58.4}\%$ & \rebuttal{$28.1$} \\ 
    \hline
    
    \multirow{3}{*}{\texttt{blue}}
    & Pre-trained & $3.65$ & $5.18$ & $0\%$ & $3.2\%$ & $7.2\%$ & - \\
    & Rewriting & $\mathbf{5.00}$ & $\mathbf{5.45}$ & $61.8\%$ & $60.2\%$ & $17.7\%$ & \rebuttal{$28.7$} \\
    & Ours (base) & $3.85$ & $5.21$ & $\mathbf{81.3}\%$ & $\mathbf{89.7}\%$ & $\mathbf{66.2}\%$ & \rebuttal{$34.7$} \\ 
    
    \bottomrule
    \end{tabular}
    \label{tab: GAN redact base}
\end{table*}

\textbf{Robustness to adversarial prompting.}
In order to understand whether adversarial prompts may cause the redacted model to generate content we would like to redact, we perform an adversarial prompting attack to redacted or rewritten models in this section. Specifically, we adopt the Square Attack \citep{andriushchenko2020square,maus2023adversarial} directly to the discrete text space. For $c\in\mC_{\Omega}$, the goal is to find an adversarial conditional $c_{\mathrm{adv}}$ such that $\mathrm{corr}(G'(z|c_{\mathrm{adv}}), c) > \mathrm{corr}(G'(z|c_{\mathrm{adv}}), \hat{c})$. The algorithm is illustrated in Algorithm \ref{alg: adv}. See Appendix \ref{appendix: text2img exp: adv} for a few examples of successful attacks. 

We measure the success rates of the proposed attack in Table \ref{tab: GAN redact adv prompt}. The success rates for our redaction method is consistently lower than the Rewriting baseline (by $31\%\sim45\%$), indicating our method is considerably more robust to adversarial prompting attacks than Rewriting. 

\begin{algorithm*}[!t]
    \centering
    \caption{Adversarial Prompting via Square Attack \citep{andriushchenko2020square,maus2023adversarial}}
    \label{alg: adv}
    \begin{algorithmic}[1]
        \STATE Initialize $c_{\mathrm{adv}}=c$.
        \FOR{iteration = $1,\cdots,16$}
        \STATE Uniformly sample a position $s$ of the caption $c_{\mathrm{adv}}$ to update.
        \STATE Uniformly sample 32 candidate words from the token dictionary. Construct 32 candidate adversarial captions by replacing the $s$-th token of $c_{\mathrm{adv}}$ with these words, respectively.
        \STATE Update the adversarial caption $c_{\mathrm{adv}}$ with the one with the largest $\mathrm{sim}(G'(z|c_{\mathrm{adv}}), c)$. 
        \ENDFOR
        \STATE \textbf{return} $c_{\mathrm{adv}}$
    \end{algorithmic}
\end{algorithm*}

\begin{table*}[!t]
    \centering
    \caption{Success rates of the adversarial prompting attack (Algorithm \ref{alg: adv}) to our redaction method and the Rewriting baseline. Our redaction method is more robust to such attacks than Rewriting.}
    
    \begin{tabular}{ccc}
    \toprule 

    Redaction prompts & Method & Attack Success Rate $(\downarrow)$ \\ \hline
    
    \multirow{2}{*}{\texttt{long beak,white belly}} 
    & Rewriting & $92.8\%$ \\ 
    & Ours (base) & $\mathbf{50.3}\%$ \\ 
    \hline
    
    \multirow{2}{*}{\texttt{blue / red wings}}
    & Rewriting & $97.4\%$ \\ 
    & Ours (base) & $\mathbf{65.7}\%$ \\ 
    \hline
    
    \multirow{2}{*}{\texttt{blue}}
    & Rewriting & $81.1\%$ \\ 
    & Ours (base) & $\mathbf{35.5}\%$ \\ 
    \hline

    \multirow{2}{*}{\texttt{yellow, red}}
    & Rewriting & $95.5\%$ \\ 
    & Ours (base) & $\mathbf{59.9}\%$ \\ 
    
    \bottomrule
    \end{tabular}
    
    \label{tab: GAN redact adv prompt}
\end{table*}

\subsection{Redacting Diffusion-based Text-to-Speech Models}

\textbf{Setup.}
We use the pre-trained DiffWave model \citep{kong2021diffwave} trained on the LJSpeech dataset \citep{Ito2017ljspeech}, which contains 13100 utterances from a female speaker reading books in home environment. The model is conditioned on Mel-spectrogram.
We redact unseen voices from the disjoint LibriTTS dataset \citep{zen2019libritts}. We randomly choose five voices to redact: speakers \texttt{125, 1578, 1737, 1926} (female's voice) and \texttt{1040} (men's voice). The training set for each voice has total lengths between 4 and 6 minutes. 

Table \ref{tab: LibriTTS split stat} includes the specific train-test splits of the LibriTTS voices. Note that for \texttt{speaker 1040} there is only one chapter id, so we split based on the segment id shown in columns. 

\begin{table*}[!t]
    \centering
    \caption{Specific train-test splits of the LibriTTS voices, and their total lengths measured in minutes.}
    
    \begin{tabular}{l|cc|cc}
    \toprule 
    \multirow{2}{*}{Redaction voices} & \multicolumn{2}{c|}{training} & \multicolumn{2}{c}{test} \\ 
    & chapter id & total length & chapter id & total length \\ \hline
    
    {\texttt{speaker 125}} & 121124 & 5.89 & 121342 & 2.30 \\ 
    {\texttt{speaker 1578}} & 140045, 140049 & 4.81 & 6379 & 1.30 \\ 
    {\texttt{speaker 1737}} & 142397, 148989, 142396 & 3.75 & 146161 & 2.51 \\
    {\texttt{speaker 1926}} &147979, 147987 & 5.44 & 143879 & 1.98 \\
    {\texttt{speaker 1040}} & 133433 (0-98) & 4.65 & 133433 (100-168) & 2.35 \\
    
    \bottomrule
    \end{tabular}
    
    \label{tab: LibriTTS split stat}
\end{table*}

\textbf{CycleGAN-VC2, Whisper, and Tortoise-TTS details}.
We train CycleGAN-VC2 \citep{kaneko2019cyclegan} with the following code \footnote{\url{https://github.com/jackaduma/CycleGAN-VC2}}.
The training data for CycleGAN-VC2 is the training data of a LibriTTS voice and the first 100 samples of LJ003 from LJSpeech \footnote{These equals $\sim1\%$ of training utterances from LJSpeech ($\sim11$ minutes).}. We train CycleGAN-VC2 for 1000 iterations with a batch size of 8. We use the medium-sized English-only Whisper model \footnote{\url{https://github.com/openai/whisper}} and the Tortoise-TTS model \footnote{\url{https://github.com/neonbjb/tortoise-tts}}. To sample from Tortoise-TTS we use two 10-second utterances from LJSpeech as the reference voice.

\textbf{Architecture and optimization.}
The architecture of the pre-trained model and other details are in Appendix \ref{appendix: tts exp}. The architecture of student conditioning networks with improved capacity is shown in Fig. \ref{fig: DiffWave H}.
We use the Adam optimizer with a learning rate $0.001$ to optimize the $\ell_1$ loss. The redaction algorithm terminates at 80000 iterations. We use a batch size of 32. Our distillation algorithm is trained with the spectrogram data only.

\textbf{Configurations.}
We first use the uniform parallel distillation loss with $\lambda=1.5$. We fix all up-sampling layers and denote it as the base configuration. We then use the spectrogram-rewriting module to improve capacity. 
Next, we improve voice cloning with Whisper and Tortoise-TTS when training CycleGAN-VC2.
Finally, we investigate non-uniform distillation losses in Table \ref{tab: tts distill non-uniform schedules}, where we set $\alpha=0.001$ and $\beta=0.01$ so that all $w_i$'s or $\lambda_i$'s have the same order or magnitude. 

\textbf{Evaluation metrics.}
To evaluate generation quality on the training voice $\mC\setminus\mC_{\Omega}$, we compute the following two speech quality metrics on the test set of LJSpeech: Perceptual Evaluation of Speech Quality (PESQ) \citep{PESQ2001} and Short-Time Objective Intelligibility (STOI) \citep{taal2011algorithm}.
To evaluate redaction quality, we train a speaker classifier between redacted and training voices in each experiment. We extract Mel-frequency cepstral coefficients \citep{xu2005hmm}, spectral contrast \citep{jiang2002music}, and chroma features \citep{ellis2007chroma} as sample features and train a support vector classifier. We then compute the recall rate of redacted voices after we perform redaction. In contrast to the standard classification, a lower recall rate means a higher fraction of redacted voices are projected to the training voice by the edited model, which indicates better redaction quality.  
See Appendix \ref{appendix: tts exp: eval} for details of these metrics.

\begin{table*}[!t]
    \centering
    \caption{Results of generation and redaction quality for redacting the man speaker \texttt{1040} in LibriTTS. The $\lambda_i$-order schedule in the non-uniform distillation losses leads to the best overall performance. The effects of each component within our method are displayed.}
    
    \begin{tabular}{ll|cc|c}
    \toprule 
    
    \multicolumn{2}{l|}{\multirow{2}{*}{Method}} & \multicolumn{2}{c|}{Speech quality (LJSpeech)} & \multirow{2}{*}{Recall ($\mC_{\Omega}$) ($\downarrow$)} \\ 
    & & PESQ ($\uparrow$) & STOI ($\uparrow$) & \\ \hline 
    \multicolumn{2}{l|}{Pre-trained} & $3.33$ & $97.8\%$ & - \\ \hline
    \multicolumn{2}{l|}{base} & $2.85$ & $95.7\%$ & $52\%$ \\ 
    \multicolumn{2}{l|}{~~+ improved capacity} & $3.03$ & $96.6\%$ & $35\%$ \\ 
    \multicolumn{2}{l|}{~~~~+ improved voice cloning} & $3.02$ & $96.6\%$ & $35\%$ \\ \hdashline
    \multirow{4}{*}{~~~~~~+ non-uniform} & $\lambda_i$-order & $\mathbf{3.23}$ & $\mathbf{97.4}\%$ & $40\%$ \\ 
    & $\lambda_i$-dilation & $3.21$ & $\mathbf{97.4}\%$ & $50\%$ \\ 
    & $w_i$-order & $3.02$ & $96.6\%$ & $\mathbf{29}\%$ \\ 
    & $w_i$-dilation & $3.02$ & $96.6\%$ & $30\%$ \\ 
    
    \bottomrule
    \end{tabular}
    \label{tab: diffwave redact config}
\end{table*}

\begin{table*}[!t]
    \centering
    \caption{Results of generation and redaction quality for redacting several female speakers in LibriTTS. The improved capacity configuration leads to the best overall performance in most settings, with an exception for speaker \texttt{1926} where both configurations lead to similar performance.}
    
    \begin{tabular}{ll|cc|c}
    \toprule 
    
    Redaction & \multirow{2}{*}{Method} & \multicolumn{2}{c|}{Speech quality (LJSpeech)} & \multirow{2}{*}{Recall ($\mC_{\Omega}$) ($\downarrow$)} \\ 
    voices & & PESQ ($\uparrow$) & STOI ($\uparrow$) & \\ \hline 

    & Pre-trained & $3.33$ & $97.8\%$ & - \\ \hline
    \multirow{2}{*}{\texttt{speaker 125}} & base & $3.14$ & $97.0\%$ & $\mathbf{0}\%$ \\ 
    & ~~+ improved capacity & $\mathbf{3.27}$ & $\mathbf{97.4}\%$ & $3\%$ \\ \hline 
    \multirow{2}{*}{\texttt{speaker 1578}} & base & $2.14$ & $94.4\%$ & $\mathbf{1}\%$ \\ 
    & ~~+ improved capacity & $\mathbf{3.24}$ & $\mathbf{97.4}\%$ & $3\%$ \\\hline 
    \multirow{2}{*}{\texttt{speaker 1737}} & base & $2.49$ & $94.9\%$ & $\mathbf{4}\%$ \\
    & ~~+ improved capacity & $\mathbf{3.24}$ & $\mathbf{97.2}\%$ & $9\%$ \\\hline 
    \multirow{2}{*}{\texttt{speaker 1926}} & base & $\mathbf{3.06}$ & $96.3\%$ & $\mathbf{16}\%$ \\
    & ~~+ improved capacity & $3.04$ & $\mathbf{96.6}\%$ & $\mathbf{16}\%$ \\
    
    \bottomrule
    \end{tabular}
    \label{tab: diffwave redact base}
\end{table*}

\textbf{Results.}
The results for redacting speaker \texttt{1040} are shown in Table \ref{tab: diffwave redact config}. With the base configuration we can redact a fraction of conditionals but the generation quality is much worse than the pre-trained model. By improving capacity both generation and redaction quality are improved. Improved voice cloning does not increase the quantitative metrics, but we find the generation quality is perceptually slightly better. The non-uniform distillation losses have a huge impact on the results. The $\lambda_i$-order and $\lambda_i$-dilation schedules can boost generation quality by a large gap without compensating redaction quality too much. The $w_i$-order and $w_i$-dilation schedules can improve redaction quality while keeping the generation quality. As high generation quality is very important for speech synthesis (on non-redacted voices), the $\lambda_i$-order schedule leads to the best overall performance.

The results for redacting other speakers are shown in Table \ref{tab: diffwave redact base}. In most settings the improved capacity configuration leads to much better generation quality than the base configuration with very little compensation for redaction quality, except for speaker \texttt{1926} where results are similar.

\textbf{Computation}.
On a single NVIDIA 3080 GPU, it takes less than 60 minutes to distill with the base configuration, and around 100 minutes with the other configurations. It takes around 2 hours to train the CycleGAN-VC2 model. As a comparison, DiffWave takes days to train on 8 GPUs. 

\textbf{Demo.}
We include audio samples in our demo website: \url{https://dataredact2023.github.io/}.

\section{Conclusion and Discussion}

In this paper, we introduce a formal \rebuttal{statistical machine learning} framework for redacting data from conditional generative models, and present a computationally efficient method that only involves the conditioning networks. We introduce explicit formula for simple models, and propose distillation-based methods for practical conditional models.
Empirically, our method performs well for practical text-to-image/speech models. It is computationally efficient, and can effectively redact certain conditionals while retaining high generation quality. 
For redacting prompts in text-to-image models, our method redacts better and is considerably more robust than the baseline methods. 
For redacting voices in text-to-speech models, our method can redact both similar and different voices while retaining high speech quality and intelligibility. 

\rebuttal{In the following we include discussion on guaranteed safety, adversarial robustness, limitations of our method, and future work.} 

\rebuttal{
\paragraph{Guaranteed Safety and Fine-tuning}
We first note that complete redaction to zero probability mass may be impossible for generative models with infinite support (as most deep generative models are). Take the unconditional normalizing flow as an example. We have the following proposition:
\begin{proposition}
    Let an invertible and smooth function $F:\mathbb{R}^d\rightarrow\mathbb{R}^d$ be an unconditional normalizing flow on the data space $\mathbb{R}^d$ that converts a standard Gaussian $\mN$ to the output distribution $F_{\#}\mN$. For any set $\mX$ that has non-zero measure on $\mathbb{R}^d$, the probability mass on $\mX$, $(F_{\#}\mN)(\mX)$, is positive.
\end{proposition}
\begin{proof}
    \begin{align*}
        (F_{\#}\mN)(\mX) & = \int_{x\in\mX}(F_{\#}\mathcal{N})(x)dx \\
        & = \int_{z\in F^{-1}(\mX)}\mathcal{N}(z)dz \\
        & = \mN(F^{-1}(\mX)).
    \end{align*}
    Because $\mX$ has positive measure and $F$ is invertible and smooth, $F^{-1}(\mX)$ has positive measure. Because $\mN$ is positive, $\mN(F^{-1}(\mX))>0$, and therefore $(F_{\#}\mN)(\mX)>0$.
\end{proof}
}

\rebuttal{
Furthermore, \citep{pham2023circumventing} discovered that fine-tuning on a set of in appropriate samples can break many mitigation methods for text-to-image models. We believe there is an \textit{impossibility result} – if the adversary have access to a dataset of inappropriate samples and fine-tune on it, there is nothing a learner can do to prevent this. In the previous normalizing flow example, the adversary can optimize the following objective: 
$$
\begin{array}{l}
\arg\max_F\mathbb{E}_{x\sim \mX}\log [(F_{\#}\mN)(x)] 
 \\
 = \arg\max_F\mathbb{E}_{z\sim F^{-1}(\mX)}\log[\mN(z)/|\mathrm{det}\nabla_z F(z)|] \\
 = \arg\min_F \mathbb{E}_{z\sim F^{-1}(\mX)}(\|z\|_2^2/2+\log|\mathrm{det}\nabla_z F(z)|) 
\end{array}
$$
and as a result having more probability mass on $\mX$.
}

\rebuttal{
Despite this impossibility result, our method has largely increased the barrier for users to generate undesirable contents. For example, if the adversary do not have enough data of a certain celebrity’s voice they are then not able to reverse engineering the network by fine-tuning on those data. In practice, it is usually necessary to combine different security mechanisms to ensure safety of generation. 
}

\rebuttal{\paragraph{Adversarial Robustness}
We have shown our method is less susceptible to be attacked by an existing adversarial prompting method than the baseline method in the text-to-image experiments. However, we would like to note that a formal definition of adversarial robustness in conditional generative models (e.g. text-to-X) is a largely open problem. We think many different threat models can be defined depending on the setting and adversary’s goal, capabilities, and knowledge of the model, which is outside the scope of this paper.
}

\rebuttal{\paragraph{Limitations}
There are certain types of neural networks that our method cannot be easily adapted to, especially when the conditional network is not completely independent from the main generative network. Examples include StyleGAN \citep{karras2020analyzing}, Transformer-based architectures with complex cross-attention layers, or multi-modal networks that mix input tokens from different modalities at the beginning. 
}

\rebuttal{
\paragraph{Future work}
One important future direction is to further improve robustness against adversarial attacks. 
Another line of future work is to apply the proposed method to Transformer-based architectures, where the conditioning networks are based on cross-attention blocks.
A third direction is to extend our method to the online setting where redacted samples come in a stream. To achieve this, we need to modify the loss in \eqref{eq: general loss} by using a weighted sampling strategy that assigns higher probability to newly seen samples.
}

\section*{Acknowledgements}
This work was supported by NSF under CNS 1804829 and ARO MURI W911NF2110317.

~~\newpage
~~\newpage
\bibliographystyle{IEEEtran}
\bibliography{main}

\newpage
\onecolumn
\appendices
\section{\rebuttal{Recovering Original Score in \citep{gandikota2023erasing}}}
\label{appendix: recovering original score}

\rebuttal{
Let $\epsilon_{\theta^*}(x_t,c,t)$ be the original score and $$\epsilon_{\theta}(x_t,c,t)=\epsilon_{\theta^*}(x_t,t)-\eta(\epsilon_{\theta^*}(x_t,c,t)-\epsilon_{\theta^*}(x_t,t))$$ be the distilled score. We could recover the original score from the distilled score in the following way. First, by letting $c=\emptyset$, we have 
$$\epsilon_{\theta}(x_t,t)=\epsilon_{\theta^*}(x_t,t).$$ 
Inserting this to the right-hand-side of the definition of distilled score, one could get 
$$\epsilon_{\theta}(x_t,c,t)=\epsilon_{\theta}(x_t,t)-\eta(\epsilon_{\theta^*}(x_t,c,t)-\epsilon_{\theta}(x_t,t)).$$ 
As a result, one can recover the original score as 
$$\epsilon_{\theta^*}(x_t,c,t)=\frac{1}{\eta}((1+\eta)\epsilon_{\theta}(x_t,t)-\epsilon_{\theta}(x_t,c,t)).$$
By using the original score for sampling one may be able to generate concepts that have been erased.}

\section{Additional Details and Experiments for Redaction from DM-GAN} \label{appendix: text2img exp}

\subsection{Details of the Pre-trained Model and the Proposed Student Networks}\label{appendix: text2img exp: detail}

The high-level architecture of DM-GAN is shown in Fig. \ref{fig: DMGAN arch} and \ref{fig: DMGAN H}. The first conditioning network $H_1$ takes the sentence embedding $v_s(c)$ as input and outputs two vectors: a mean vector, and the square root of the variance vector. A re-parameterization similar to variational auto-encoders is applied to these two vectors, and the output is concatenated to the latent code. The other two conditioning networks $H_2$ and $H_3$, called the memory writing module, take two inputs: the word embeddings $v_w(c)$, and the image features of the previously generated low resolution images. The output of $H_2$ or $H_3$ then goes through the rest of the modules in the main generative network.
We use the pre-trained model and code from \url{https://github.com/MinfengZhu/DM-GAN} under MIT license. The pre-trained model takes days to train on 1 or more GPUs. 

\begin{figure}[!h]
    \centering
    \includegraphics[width=0.8\textwidth]{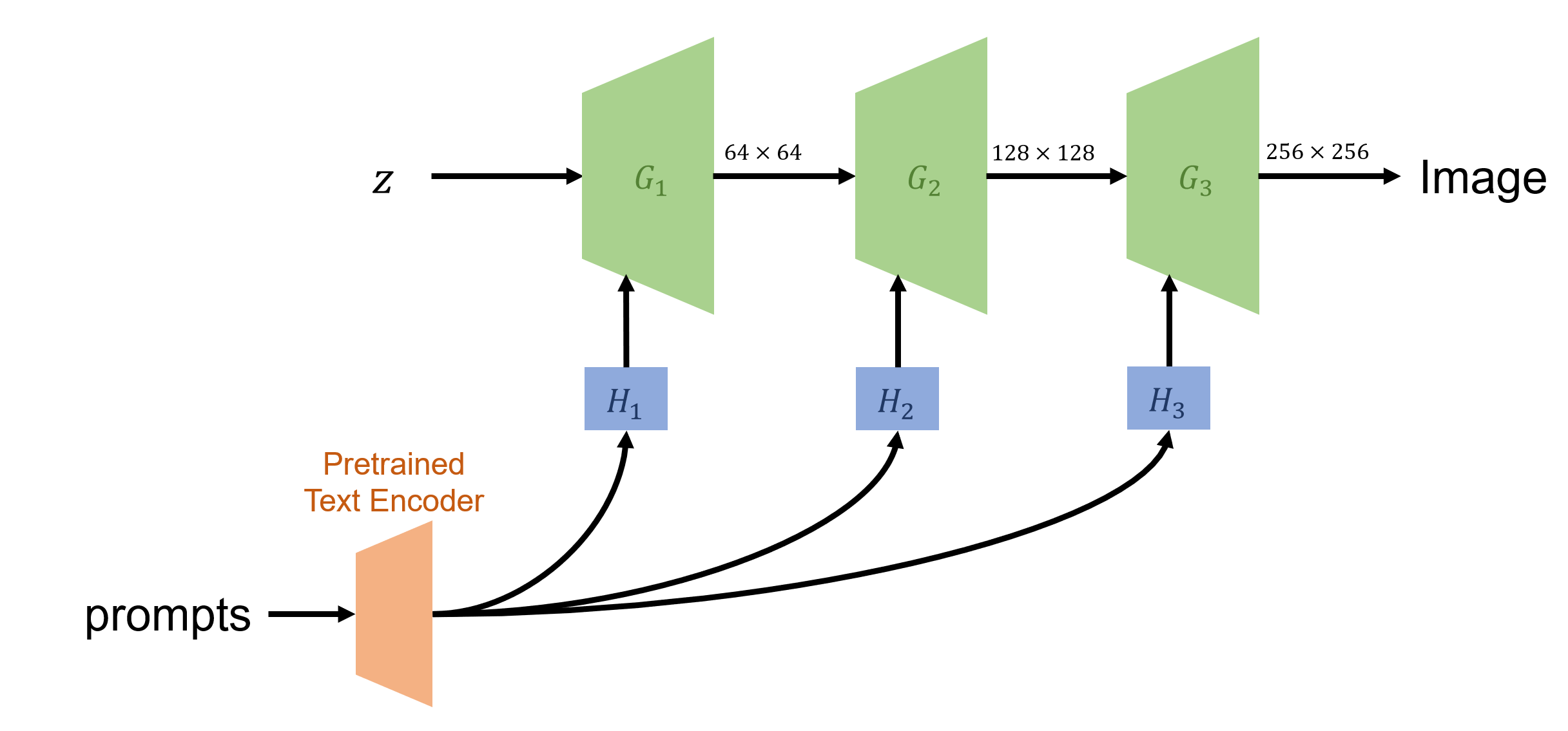}
    \caption{High-level architecture of DM-GAN.}
    \label{fig: DMGAN arch}
\end{figure}

\begin{figure}[!h]
    \centering
    \includegraphics[width=0.9\textwidth]{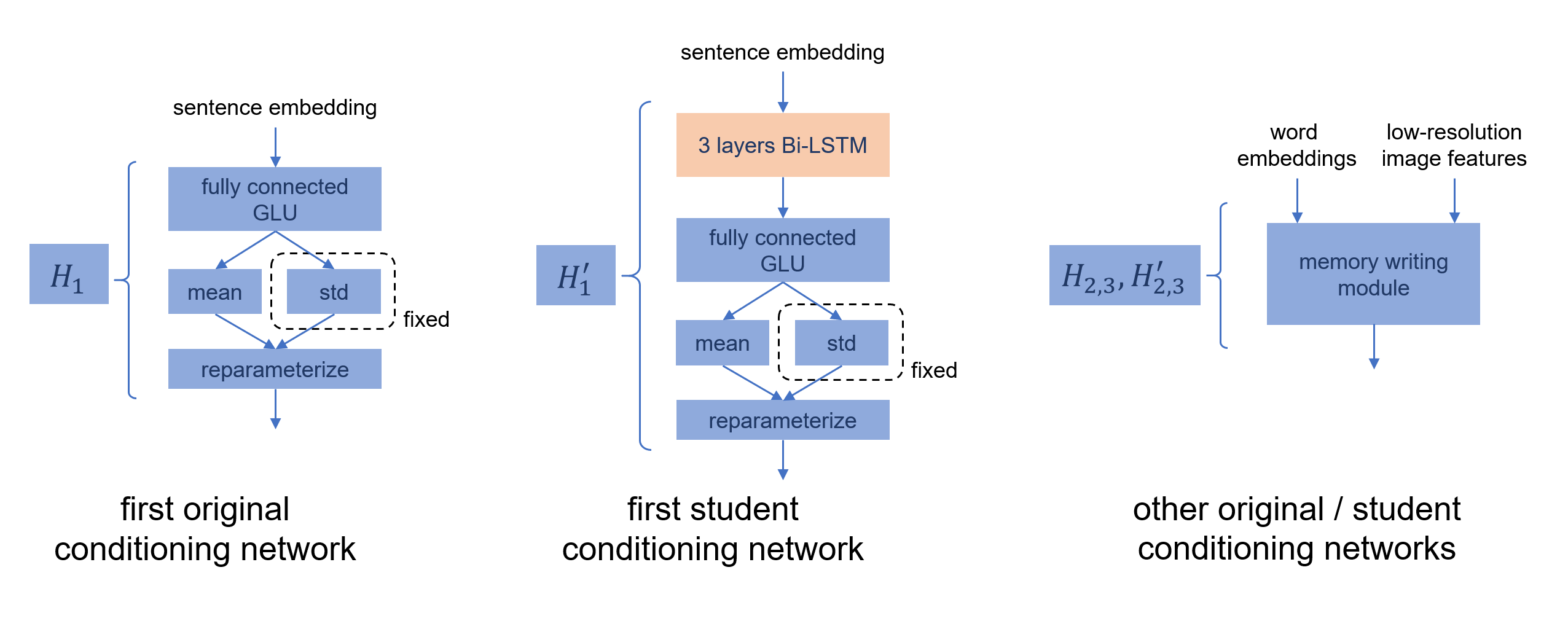}
    \caption{High-level architecture of original and higher-capacity conditioning networks of DM-GAN.}
    \label{fig: DMGAN H}
\end{figure}

\subsection{Visualization}\label{appendix: text2img exp: vis}

In Fig. \ref{fig: redact long beak white belly 1} - Fig. \ref{fig: redact long beak white belly 4} , we visualize examples where we redact prompts that contain \texttt{long beak} or \texttt{white belly}.
In Fig. \ref{fig: redact blue red wings 1} - Fig. \ref{fig: redact blue red wings 4} , we visualize examples where we redact prompts that contain \texttt{blue wings} or \texttt{red wings}.
In Fig. \ref{fig: redact blue 1} - Fig. \ref{fig: redact blue 4} , we visualize examples where we redact prompts that contain \texttt{blue}.
In Fig. \ref{fig: redact yellow red 1} - Fig. \ref{fig: redact yellow red 4} , we visualize examples where we redact prompts that contain \texttt{yellow} or \texttt{red}.

~~\newpage

\renewcommand{\promptsubfig}[4]{
	\begin{subfigure}[b]{0.228\textwidth}
		\centering 
		\includegraphics[width=0.99\textwidth]{figs-#1-#2-#3}
		\caption{#4}
	\end{subfigure}
}

\renewcommand{\promptsetting}{base-long_beak_white_belly_to_short_beak_black_belly-invalid}
\renewcommand{\promptbaseline}{Rewrite-upsample1-long_beak_white_belly_to_short_beak_black_belly-invalid}
\renewcommand{\promptid}{006.Least_Auklet_sentenceID34-9}

\begin{figure}[!h]
\centering
	\promptsubfig{\promptsetting}{\promptid}{pretrained.png}{Pre-trained $G(\cdot|c)$}
	\promptsubfig{\promptsetting}{\promptid}{target.png}{Reference $G(\cdot|\hat{c})$}
	\promptsubfig{\promptsetting}{\promptid}{distilled.png}{Our Redaction $G'(\cdot|c)$}
	\promptsubfig{\promptbaseline}{\promptid}{distilled.png}{Rewriting Baseline}
	\caption{\small Redacted prompt: \texttt{``this particular bird has a white belly and breasts and black head and back''}. Reference prompt: \texttt{``this particular bird has a black belly and breasts and black head and back''}.}
	\label{fig: redact long beak white belly 1}
\end{figure}

\renewcommand{\promptid}{031.Black_billed_Cuckoo_sentenceID227-0}
\begin{figure}[!h]
\centering
        \promptsubfig{\promptsetting}{\promptid}{pretrained.png}{Pre-trained $G(\cdot|c)$}
	\promptsubfig{\promptsetting}{\promptid}{target.png}{Reference $G(\cdot|\hat{c})$}
	\promptsubfig{\promptsetting}{\promptid}{distilled.png}{Our Redaction $G'(\cdot|c)$}
	\promptsubfig{\promptbaseline}{\promptid}{distilled.png}{Rewriting Baseline}
	\caption{\small Redacted prompt: \texttt{``this bird has feathers that are black and has a white belly''}. Reference prompt: \texttt{``this bird has feathers that are black and has a black belly''}.}
	\label{fig: redact long beak white belly 2}
\end{figure}

\renewcommand{\promptid}{036.Northern_Flicker_sentenceID539-3}
\begin{figure}[!h]
\centering
    	\promptsubfig{\promptsetting}{\promptid}{pretrained.png}{Pre-trained $G(\cdot|c)$}
	\promptsubfig{\promptsetting}{\promptid}{target.png}{Reference $G(\cdot|\hat{c})$}
	\promptsubfig{\promptsetting}{\promptid}{distilled.png}{Our Redaction $G'(\cdot|c)$}
	\promptsubfig{\promptbaseline}{\promptid}{distilled.png}{Rewriting Baseline}
	\caption{\small Redacted prompt: \texttt{``a small bird with an orange throat and long beak''}. Reference prompt: \texttt{``a small bird with an orange throat and short beak''}.}
	\label{fig: redact long beak white belly 3}
\end{figure}

\renewcommand{\promptid}{079.Belted_Kingfisher_sentenceID1229-4}
\begin{figure}[!h]
\centering
    	\promptsubfig{\promptsetting}{\promptid}{pretrained.png}{Pre-trained $G(\cdot|c)$}
	\promptsubfig{\promptsetting}{\promptid}{target.png}{Reference $G(\cdot|\hat{c})$}
	\promptsubfig{\promptsetting}{\promptid}{distilled.png}{Our Redaction $G'(\cdot|c)$}
	\promptsubfig{\promptbaseline}{\promptid}{distilled.png}{Rewriting Baseline}
	\caption{\small Redacted prompt: \texttt{``the black and white bird has a sharp long beak''}. Reference prompt: \texttt{``the black and white bird has a sharp short beak''}.}
	\label{fig: redact long beak white belly 4}
\end{figure}

~~\newpage
\renewcommand{\promptsetting}{base-blue_red_wings_to_white_wings-invalid}
\renewcommand{\promptbaseline}{Rewrite-upsample1-blue_red_wings_to_white_wings-invalid}
\renewcommand{\promptid}{001.Black_footed_Albatross_sentenceID8-0}

\begin{figure}[!h]
\centering
	\promptsubfig{\promptsetting}{\promptid}{pretrained.png}{Pre-trained $G(\cdot|c)$}
	\promptsubfig{\promptsetting}{\promptid}{target.png}{Reference $G(\cdot|\hat{c})$}
	\promptsubfig{\promptsetting}{\promptid}{distilled.png}{Our Redaction $G'(\cdot|c)$}
	\promptsubfig{\promptbaseline}{\promptid}{distilled.png}{Rewriting Baseline}
	\caption{\small Redacted prompt: \texttt{``this bird has wings that are blue and has black feet''}. Reference prompt: \texttt{``this bird has wings that are white and has black feet''}.}
	\label{fig: redact blue red wings 1}
\end{figure}

\renewcommand{\promptid}{009.Brewer_Blackbird_sentenceID4-2}
\begin{figure}[!h]
\centering
	\promptsubfig{\promptsetting}{\promptid}{pretrained.png}{Pre-trained $G(\cdot|c)$}
	\promptsubfig{\promptsetting}{\promptid}{target.png}{Reference $G(\cdot|\hat{c})$}
	\promptsubfig{\promptsetting}{\promptid}{distilled.png}{Our Redaction $G'(\cdot|c)$}
	\promptsubfig{\promptbaseline}{\promptid}{distilled.png}{Rewriting Baseline}
	\caption{\small Redacted prompt: \texttt{``this is a grey bird with blue wings and a pointy beak''}. Reference prompt: \texttt{``this is a grey bird with white wings and a pointy beak''}.}
	\label{fig: redact blue red wings 2}
\end{figure}

\renewcommand{\promptid}{035.Purple_Finch_sentenceID56-7}
\begin{figure}[!h]
\centering
	\promptsubfig{\promptsetting}{\promptid}{pretrained.png}{Pre-trained $G(\cdot|c)$}
	\promptsubfig{\promptsetting}{\promptid}{target.png}{Reference $G(\cdot|\hat{c})$}
	\promptsubfig{\promptsetting}{\promptid}{distilled.png}{Our Redaction $G'(\cdot|c)$}
	\promptsubfig{\promptbaseline}{\promptid}{distilled.png}{Rewriting Baseline}
	\caption{\small Redacted prompt: \texttt{``this bird has wings that are red and has a white belly''}. Reference prompt: \texttt{``this bird has wings that are white and has a white belly''}.}
	\label{fig: redact blue red wings 3}
\end{figure}

\renewcommand{\promptid}{186.Cedar_Waxwing_sentenceID299-4}
\begin{figure}[!h]
\centering
	\promptsubfig{\promptsetting}{\promptid}{pretrained.png}{Pre-trained $G(\cdot|c)$}
	\promptsubfig{\promptsetting}{\promptid}{target.png}{Reference $G(\cdot|\hat{c})$}
	\promptsubfig{\promptsetting}{\promptid}{distilled.png}{Our Redaction $G'(\cdot|c)$}
	\promptsubfig{\promptbaseline}{\promptid}{distilled.png}{Rewriting Baseline}
	\caption{\small Redacted prompt: \texttt{``this bird has wings that are red and has a yellow belly''}. Reference prompt: \texttt{``this bird has wings that are white and has a yellow belly''}.}
	\label{fig: redact blue red wings 4}
\end{figure}

~~\newpage
\renewcommand{\promptsetting}{base-blue_to_red-invalid}
\renewcommand{\promptbaseline}{Rewrite-upsample1-blue_to_red-invalid}
\renewcommand{\promptid}{001.Black_footed_Albatross_sentenceID9-1}

\begin{figure}[!h]
\centering
	\promptsubfig{\promptsetting}{\promptid}{pretrained.png}{Pre-trained $G(\cdot|c)$}
	\promptsubfig{\promptsetting}{\promptid}{target.png}{Reference $G(\cdot|\hat{c})$}
	\promptsubfig{\promptsetting}{\promptid}{distilled.png}{Our Redaction $G'(\cdot|c)$}
	\promptsubfig{\promptbaseline}{\promptid}{distilled.png}{Rewriting Baseline}
	\caption{\small Redacted prompt: \texttt{``this bird has wings that are blue and has black feet''}. Reference prompt: \texttt{``this bird has wings that are red and has black feet''}.}
	\label{fig: redact blue 1}
\end{figure}

\renewcommand{\promptid}{008.Rhinoceros_Auklet_sentenceID29-4}
\begin{figure}[!h]
\centering
	\promptsubfig{\promptsetting}{\promptid}{pretrained.png}{Pre-trained $G(\cdot|c)$}
	\promptsubfig{\promptsetting}{\promptid}{target.png}{Reference $G(\cdot|\hat{c})$}
	\promptsubfig{\promptsetting}{\promptid}{distilled.png}{Our Redaction $G'(\cdot|c)$}
	\promptsubfig{\promptbaseline}{\promptid}{distilled.png}{Rewriting Baseline}
	\caption{\small Redacted prompt: \texttt{``this bird has small wings and blue grey nape''}. Reference prompt: \texttt{``this bird has small wings and red grey nape''}.}
	\label{fig: redact blue 2}
\end{figure}

\renewcommand{\promptid}{014.Indigo_Bunting_sentenceID129-4}
\begin{figure}[!h]
\centering
	\promptsubfig{\promptsetting}{\promptid}{pretrained.png}{Pre-trained $G(\cdot|c)$}
	\promptsubfig{\promptsetting}{\promptid}{target.png}{Reference $G(\cdot|\hat{c})$}
	\promptsubfig{\promptsetting}{\promptid}{distilled.png}{Our Redaction $G'(\cdot|c)$}
	\promptsubfig{\promptbaseline}{\promptid}{distilled.png}{Rewriting Baseline}
	\caption{\small Redacted prompt: \texttt{``the bird is blue with gray wins and tail''}. Reference prompt: \texttt{``the bird is red with gray wins and tail''}.}
	\label{fig: redact blue 3}
\end{figure}

\renewcommand{\promptid}{079.Belted_Kingfisher_sentenceID1095-3}
\begin{figure}[!h]
\centering
	\promptsubfig{\promptsetting}{\promptid}{pretrained.png}{Pre-trained $G(\cdot|c)$}
	\promptsubfig{\promptsetting}{\promptid}{target.png}{Reference $G(\cdot|\hat{c})$}
	\promptsubfig{\promptsetting}{\promptid}{distilled.png}{Our Redaction $G'(\cdot|c)$}
	\promptsubfig{\promptbaseline}{\promptid}{distilled.png}{Rewriting Baseline}
	\caption{\small Redacted prompt: \texttt{``this bird has wings that are blue and has a white belly''}. Reference prompt: \texttt{``this bird has wings that are red and has a white belly''}.}
	\label{fig: redact blue 4}
\end{figure}

~~\newpage
\renewcommand{\promptsetting}{LSTM_fixvar_anneal_1_3-yellow_red_to_black_black-invalid}
\renewcommand{\promptbaseline}{Rewrite-upsample1-yellow_red_to_black_black-invalid}
\renewcommand{\promptid}{035.Purple_Finch_sentenceID986-0}

\begin{figure}[!h]
\centering
	\promptsubfig{\promptsetting}{\promptid}{pretrained.png}{Pre-trained $G(\cdot|c)$}
	\promptsubfig{\promptsetting}{\promptid}{target.png}{Reference $G(\cdot|\hat{c})$}
	\promptsubfig{\promptsetting}{\promptid}{distilled.png}{Our Redaction $G'(\cdot|c)$}
	\promptsubfig{\promptbaseline}{\promptid}{distilled.png}{Rewriting Baseline}
	\caption{\small Redacted prompt: \texttt{``this is a red bird with a white belly and a large beak''}. Reference prompt: \texttt{``this is a black bird with a white belly and a large beak''}.}
	\label{fig: redact yellow red 1}
\end{figure}

\renewcommand{\promptid}{035.Purple_Finch_sentenceID991-3}
\begin{figure}[!h]
\centering
	\promptsubfig{\promptsetting}{\promptid}{pretrained.png}{Pre-trained $G(\cdot|c)$}
	\promptsubfig{\promptsetting}{\promptid}{target.png}{Reference $G(\cdot|\hat{c})$}
	\promptsubfig{\promptsetting}{\promptid}{distilled.png}{Our Redaction $G'(\cdot|c)$}
	\promptsubfig{\promptbaseline}{\promptid}{distilled.png}{Rewriting Baseline}
	\caption{\small Redacted prompt: \texttt{``a bird with thick short beak red crown red breast that fades into a pink and white belly and red coverts''}. Reference prompt: \texttt{``a bird with thick short beak black crown black breast that fades into a pink and white belly and black coverts''}.}
	\label{fig: redact yellow red 2}
\end{figure}

\renewcommand{\promptid}{038.Great_Crested_Flycatcher_sentenceID1910-2}
\begin{figure}[!h]
\centering
	\promptsubfig{\promptsetting}{\promptid}{pretrained.png}{Pre-trained $G(\cdot|c)$}
	\promptsubfig{\promptsetting}{\promptid}{target.png}{Reference $G(\cdot|\hat{c})$}
	\promptsubfig{\promptsetting}{\promptid}{distilled.png}{Our Redaction $G'(\cdot|c)$}
	\promptsubfig{\promptbaseline}{\promptid}{distilled.png}{Rewriting Baseline}
	\caption{\small Redacted prompt: \texttt{``this yellow breasted bird has a dark gray head and chest a thin beak and a long tail''}. Reference prompt: \texttt{``this black breasted bird has a dark gray head and chest a thin beak and a long tail''}.}
	\label{fig: redact yellow red 3}
\end{figure}

\renewcommand{\promptid}{043.Yellow_bellied_Flycatcher_sentenceID2189-1}
\begin{figure}[!h]
\centering
	\promptsubfig{\promptsetting}{\promptid}{pretrained.png}{Pre-trained $G(\cdot|c)$}
	\promptsubfig{\promptsetting}{\promptid}{target.png}{Reference $G(\cdot|\hat{c})$}
	\promptsubfig{\promptsetting}{\promptid}{distilled.png}{Our Redaction $G'(\cdot|c)$}
	\promptsubfig{\promptbaseline}{\promptid}{distilled.png}{Rewriting Baseline}
	\caption{\small Redacted prompt: \texttt{``the yellow and black bird has a white belly''}. Reference prompt: \texttt{``the black and black bird has a white belly''}.}
	\label{fig: redact yellow red 4}
\end{figure}

~~\newpage
~~\newpage
\subsection{Adversarial Prompting Attack} \label{appendix: text2img exp: adv}

\renewcommand{\promptsubfig}[4]{
	\begin{subfigure}[b]{0.235\textwidth}
		\centering 
		\includegraphics[width=0.99\textwidth]{figs-#1-#2-#3}
		\caption{#4}
	\end{subfigure}
}
\renewcommand{\promptsetting}{adv-LSTM_fixvar_anneal_1_3-yellow_red_to_black_black-invalid}
\renewcommand{\promptid}{008.Rhinoceros_Auklet_sentenceID85}
\begin{figure}[!h]
\centering
	\promptsubfig{\promptsetting}{\promptid}{pretrained.png}{Pre-trained $G(\cdot|c)$}
	\promptsubfig{\promptsetting}{\promptid}{distilled.png}{Our Redaction $G'(\cdot|c)$}
	\promptsubfig{\promptsetting}{\promptid}{adversarial.png}{Adversarial $G'(\cdot|c_{\mathrm{adv}})$}
	\caption{\small Redacted prompt: \texttt{``a short yellow beak on an otherwise black bird bright irises around black pupils''}. Adversarial prompt: \texttt{``a short yellow wingtip on yhellow otherwise wa bird bright irises around dramatic pupils''}.}
	\label{fig: adv 1}
\end{figure}

\renewcommand{\promptid}{035.Purple_Finch_sentenceID957}
\begin{figure}[!h]
\centering
	\promptsubfig{\promptsetting}{\promptid}{pretrained.png}{Pre-trained $G(\cdot|c)$}
	\promptsubfig{\promptsetting}{\promptid}{distilled.png}{Our Redaction $G'(\cdot|c)$}
	\promptsubfig{\promptsetting}{\promptid}{adversarial.png}{Adversarial $G'(\cdot|c_{\mathrm{adv}})$}
	\caption{\small Redacted prompt: \texttt{``this is a colorful bird boasting a coral red color strongest at head and fading towards tail''}. Adversarial prompt: \texttt{``wingtip is a colorful bird boasting the coral red color strongest at primary and fading towards tail''}.}
	\label{fig: adv 2}
\end{figure}

\renewcommand{\promptid}{043.Yellow_bellied_Flycatcher_sentenceID2304}
\begin{figure}[!h]
\centering
	\promptsubfig{\promptsetting}{\promptid}{pretrained.png}{Pre-trained $G(\cdot|c)$}
	\promptsubfig{\promptsetting}{\promptid}{distilled.png}{Our Redaction $G'(\cdot|c)$}
	\promptsubfig{\promptsetting}{\promptid}{adversarial.png}{Adversarial $G'(\cdot|c_{\mathrm{adv}})$}
	\caption{\small Redacted prompt: \texttt{``this small bird has a light yellow belly and a dark yellow head with black wings and tail''}. Adversarial prompt: \texttt{``this small bird the the light yellow belly and a dark yellow head with black remaining and mahogany''}.}
	\label{fig: adv 3}
\end{figure}

\renewcommand{\promptid}{095.Baltimore_Oriole_sentenceID4106}
\begin{figure}[!h]
\centering
	\promptsubfig{\promptsetting}{\promptid}{pretrained.png}{Pre-trained $G(\cdot|c)$}
	\promptsubfig{\promptsetting}{\promptid}{distilled.png}{Our Redaction $G'(\cdot|c)$}
	\promptsubfig{\promptsetting}{\promptid}{adversarial.png}{Adversarial $G'(\cdot|c_{\mathrm{adv}})$}
	\caption{\small Redacted prompt: \texttt{``this bird is golden yellow with a black head''}. Adversarial prompt: \texttt{``through mahogany wingtip golden yellow bed a black head''}.}
	\label{fig: adv 4}
\end{figure}

\renewcommand{\promptid}{096.Hooded_Oriole_sentenceID4283}
\begin{figure}[!h]
\centering
	\promptsubfig{\promptsetting}{\promptid}{pretrained.png}{Pre-trained $G(\cdot|c)$}
	\promptsubfig{\promptsetting}{\promptid}{distilled.png}{Our Redaction $G'(\cdot|c)$}
	\promptsubfig{\promptsetting}{\promptid}{adversarial.png}{Adversarial $G'(\cdot|c_{\mathrm{adv}})$}
	\caption{\small Redacted prompt: \texttt{``this bird has a golden yellow body a black head a pink bill that curves downward''}. Adversarial prompt: \texttt{``mahogany the has blacker golden yellow basically a black head a primary bill that through hughed''}.}
	\label{fig: adv 5}
\end{figure}

\renewcommand{\promptid}{098.Scott_Oriole_sentenceID4774}
\begin{figure}[!h]
\centering
	\promptsubfig{\promptsetting}{\promptid}{pretrained.png}{Pre-trained $G(\cdot|c)$}
	\promptsubfig{\promptsetting}{\promptid}{distilled.png}{Our Redaction $G'(\cdot|c)$}
	\promptsubfig{\promptsetting}{\promptid}{adversarial.png}{Adversarial $G'(\cdot|c_{\mathrm{adv}})$}
	\caption{\small Redacted prompt: \texttt{``this bird has a black beak a black crown and a belly that is golden yellow''}. Adversarial prompt: \texttt{``mahogany love has a riangular love bed black crown and a dramatic hjas breadt golden yellow''}.}
	\label{fig: adv 6}
\end{figure}

~~\newpage
~~\newpage
\section{Additional Details and Experiments for Redaction from DiffWave} \label{appendix: tts exp}

\subsection{Details of the Pre-trained Model and the Proposed Student Networks} \label{appendix: tts exp: detail}

The high-level architecture of DiffWave is shown in Fig. \ref{fig: DiffWave arch}. We select the base (64 channels) version of the model. The model is conditioned on 80-band Mel-spectrogram with FFT size$=1024$, hop size$=256$, and window size$=1024$. Each conditioning network has two up-sampling layers that up-sample the spectrogram, and a one-dimensional convolution layer that maps the number of channels to 128. We use the pre-trained model and code from \url{https://github.com/philsyn/DiffWave-Vocoder} under MIT license, which is trained on all LJSpeech samples except for LJ001 and LJ002, which is used as the test set. The pre-trained model takes days to train on 8 GPUs.

\begin{figure}[!h]
    \centering
    \includegraphics[width=0.8\textwidth]{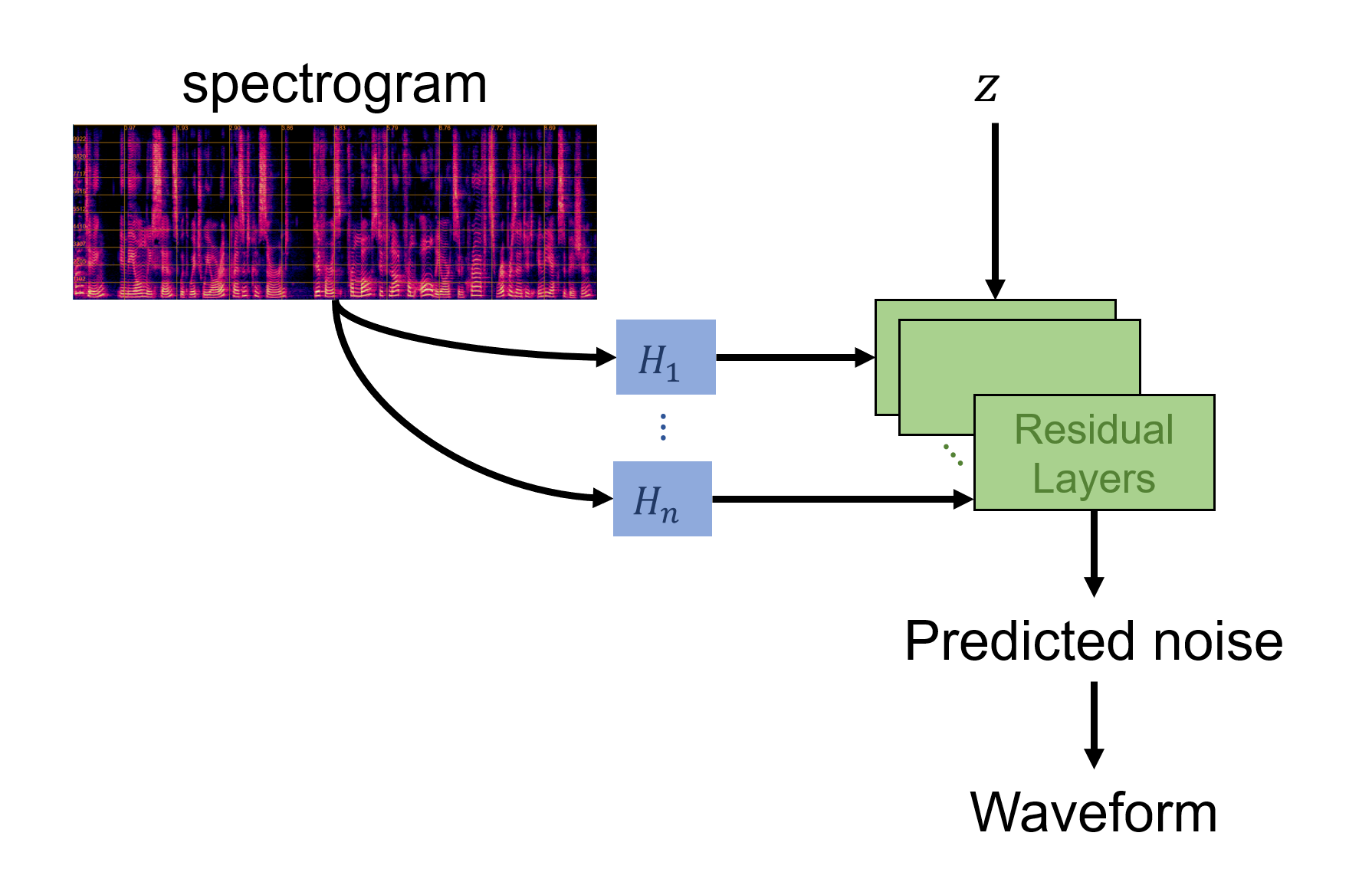}
    \caption{High-level architecture of DiffWave.}
    \label{fig: DiffWave arch}
\end{figure}

For the additional layers in the improved capacity configuration, all convolutions are one-dimensional with kernel size $=1$. $h_{\mathrm{trans}}$ includes two convolutions that keep the channels ($=80$) and a leaky ReLU activation with negative slope $=0.4$ between. $h_{\mathrm{gate}}$ includes one zero-initialized convolution that changes channels from 80 to 128 followed by a sigmoid activation. The architecture of student conditioning networks with improved capacity is shown in Fig. \ref{fig: DiffWave H}.

\begin{figure}[!h]
    \centering
    \includegraphics[width=0.9\textwidth]{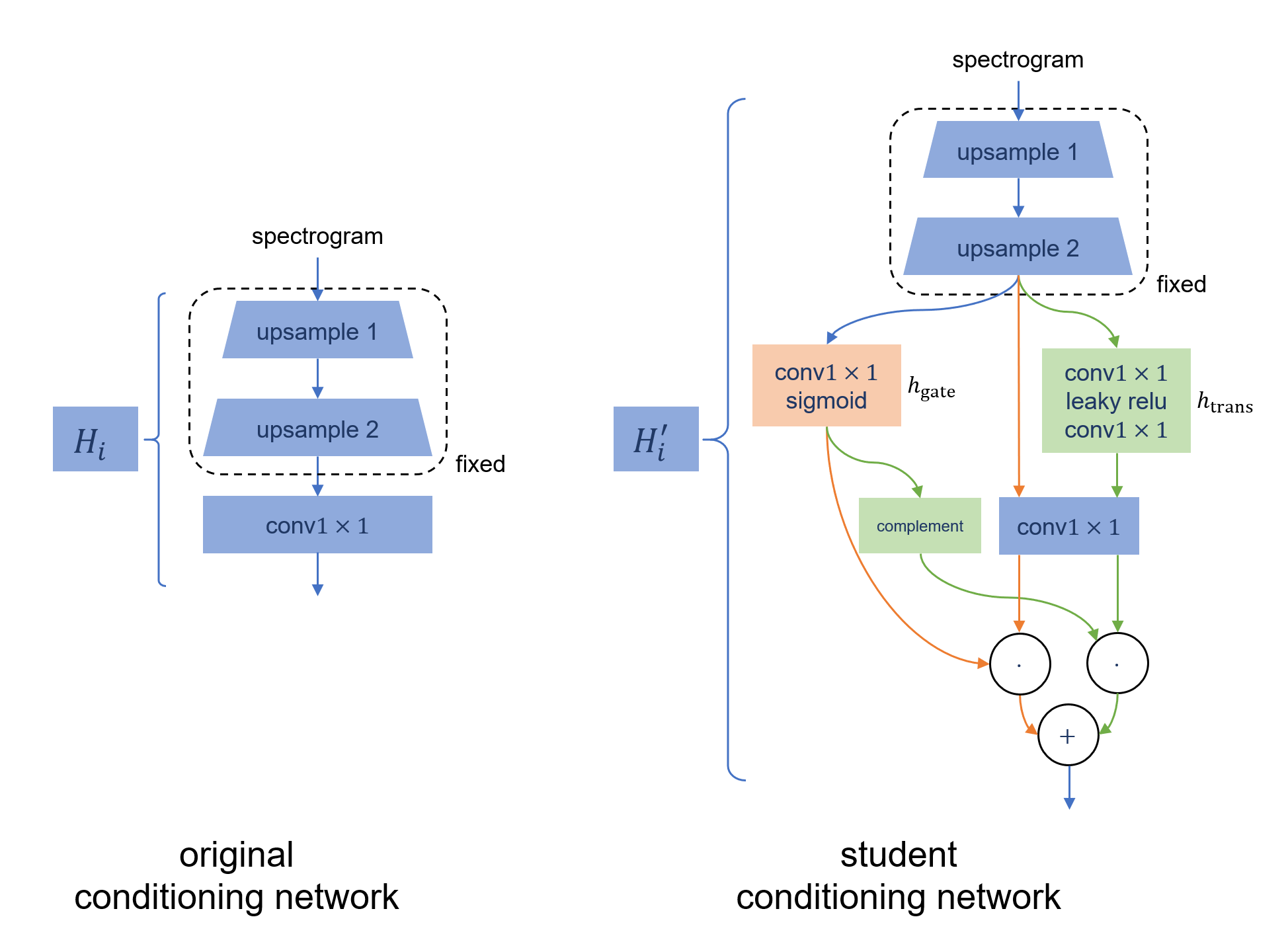}
    \caption{High-level architecture of original and higher-capacity conditioning networks of DiffWave.}
    \label{fig: DiffWave H}
\end{figure}

\subsection{Evaluation Metrics} \label{appendix: tts exp: eval}

The metrics for speech quality are as follows.
\begin{enumerate}
    \item Perceptual Evaluation of Speech Quality \citep{PESQ2001}, or PESQ, measures the quality of generated speech. It ranges between -0.5 and 4.5 and is higher for better quality.
    \item Short-Time Objective Intelligibility \citep{taal2011algorithm}, or STOI, measures the intelligibility of generated speech. It ranges between 0\% and 100\% and is higher for better intelligibility.
\end{enumerate}

The voice classifier is trained and tested on audio clips with 0.7256 second. For each audio clip, we extract 20-dimensional Mel-frequency cepstral
coefficients \citep{xu2005hmm}, 7-dimensional spectral contrast \citep{jiang2002music}, and 12-dimensional chroma features \citep{ellis2007chroma}. The classifier is a support vector classifier with the radial basis function kernel with regularization coefficient $=1$.

\end{document}